\documentclass[a4,12pt]{article}

\usepackage{amsmath,amsthm,amssymb}
\usepackage{algorithm,algorithmic}
\usepackage[truedimen,margin=20truemm]{geometry}
\usepackage{comment}
\usepackage{microtype}
\usepackage{graphicx}
\usepackage{subfigure}
\usepackage{booktabs} % for professional tables
\usepackage{tikz,url,multirow}
\usetikzlibrary{positioning}

\theoremstyle{plain}
\newtheorem{thm}{Theorem}
\newtheorem{lem}[thm]{Lemma}
\newtheorem{prop}[thm]{Proposition}
\newtheorem{cor}[thm]{Corollary}

\theoremstyle{definition}

\newtheorem{rmk}[thm]{Remark}

\newcommand{\Prob}{\mathrm{Prob}}

\newcommand{\cA}{\mathcal{A}}

\newcommand{\cP}{\mathcal{P}}

\newcommand{\rI}{\mathrm{I}}

\newcommand{\cS}{\mathcal{S}}

\newcommand{\cU}{\mathcal{U}}
\newcommand{\cV}{\mathcal{V}}

\newcommand{\cY}{\mathcal{Y}}

\newcommand{\bN}{\mathbf{N}}

\newcommand{\bR}{\mathbf{R}}

\newcommand{\rdo}{\mathrm{do}}

\newcommand{\rcat}{\mathrm{cat}}

\newcommand{\argmax}{\mathop{\rm argmax}\limits}

\def\v#1{\mbox{\boldmath $#1$}}

\title{Causal Bandits with Propagating Inference}
\author{Akihiro Yabe \\ NEC Corporation \\ a-yabe@cq.jp.nec.com \and Daisuke Hatano \\ Riken AIP \\ daisuke.hatano@riken.jp  \and Hanna Sumita \\ Tokyo Metropolitan University \\ sumita@tmu.ac.jp \and Shinji Ito \\ NEC Corporation \\ s-ito@me.jp.nec.com  \and Naonori Kakimura \\ Keio University \\ kakimura@math.keio.ac.jp \and Takuro Fukunaga \\ Riken AIP \\ takuro.fukunaga@riken.jp  \and Ken-ichi Kawarabayashi \\ National Institute of Informatics \\ k\_keniti@nii.ac.jp}
\date{}
\begin{document}
\maketitle

\renewcommand{\thefootnote}{\fnsymbol{footnote}}
\footnote[0]{To appear in International Conference on Machine Learning 2018.}
\renewcommand{\thefootnote}{\arabic{footnote}}

\begin{abstract}
	Bandit is a framework for designing sequential experiments.
	In each experiment, a learner selects an arm $A \in \cA$
	and obtains an observation corresponding to $A$.
	Theoretically, the tight regret lower-bound for the general bandit
	is polynomial with respect to the number of arms $|\cA|$.
	This makes bandit incapable of handling an exponentially large number of arms,
	hence the bandit problem with side-information is often considered to overcome this lower bound.
	Recently, a bandit framework over a causal graph was introduced, where the structure of the causal graph is available as side-information.
	A causal graph is a fundamental model that is frequently used with a variety of real problems.
	In this setting, the arms are identified with
	interventions on a given causal graph, and the effect of an intervention propagates throughout all over the causal graph.
	The task is to find the best intervention
	that maximizes the expected value on a target node.
	Existing algorithms for causal bandit overcame the
	$\Omega(\sqrt{|\cA|/T})$ simple-regret lower-bound;
	however, their algorithms work only when the interventions $\cA$
	are localized around a single node (i.e., an intervention propagates only to its neighbors).
	
	We propose a novel causal bandit algorithm for an arbitrary set of interventions, which can propagate throughout the causal graph.
	We also show that it achieves $O(\sqrt{ \gamma^*\log(|\cA|T) / T})$ regret bound,
	where $\gamma^*$ is determined by using a causal graph structure.
	In particular, if the in-degree of the causal graph is bounded, 
	then $\gamma^* = O(N^2)$, where $N$ is the number of nodes.
\end{abstract}

\section{Introduction}
Multi-armed bandit has been widely recognized as
a standard framework for modeling online learning with a limited number of observations.
In each round in the bandit problem, a learner chooses an arm $A$ from given candidates $\cA$,
and obtains a corresponding observation.
Since observation is limited, %on the selected arm,
the learner must adopt an efficient strategy for exploring the optimal arm $A^* \in \cA$.
The efficiency of the strategy is measured by regret,
and the theoretically tight lower-bound is $O(\sqrt{|\cA|})$ with respect to the number of arms $|\cA|$ in the general multi-armed bandit setting. Thus, in order to improve the above lower bound, one requires additional information for the bandit setting.
For example, contextual bandit~\cite{agarwal2014taming,auer2002nonstochastic} is a well-known class of bandit problems with side information on domain-expert knowledge. For this setting,
there is a logarithmic regret bound $O(\sqrt{\log |\cA| })$ with respect to the number of arms.
In this paper, we also achieve $O(\sqrt{\log |\cA| })$ regret bound for a novel class of bandit problems with side information. To this end, let us introduce our bandit setting in detail.

Causal graph~\cite{pearl2009causality} is a well-known tool for modeling a variety of real problems,
including computational advertising~\cite{bottou2013counterfactual}, genetics~\cite{meinshausen2016methods}, agriculture~\cite{splawa1990application}, and marketing~\cite{kim2008trust}.
%Given causal modeling,
%on the basis of
Based on
causal graph discovery studies~\cite{eberhardt2005number,hauser2014two,hu2014randomized,shanmugam2015learning},
Lattimore et al.~\cite{lattimore2016causal} recently introduced
the causal bandit framework. They consider the problem of finding the best intervention which causes desirable propagation of
a probabilistic distribution over a given causal graph with a limited number of experiments $T$.
In this setting, the arms %in thhe general multi-armed bandit problem
are identified
as interventions $\cA$ on the causal graph.
A set of binary random variables $V_1,V_2,\dots,V_N$ is associated with nodes $v_1,v_2,\dots,v_N$ of the causal graph.
At each round of an experiment, a learner selects an intervention $A \in \cA \subseteq \{0,1,*\}^N$ which enforces a
realization of a variable $V_i$ to $A_i$ when $A_i \in \{0,1\}$.
%realization of a subset $I_A \subseteq [1,N]$ of the variables with assigned value $A_i \in \{0,1\}$.
The effect of the intervention then propagates throughout the causal graph through the edges,
%the adjacency of nodes,
and a realization $\omega \in \{0,1 \}^N$ over all nodes is observed after propagation.
The goal of the causal bandit problem is to control the realization of a target variable $V_N$
with an optimal intervention.
%See Figure~\ref{figure_ex_graph} for an illustrative example.

Figure~\ref{figure_ex_graph} is an illustrative example of the causal bandit problem.
In the figure, the four nodes on the right represent a consumer decision-making model in e-commerce borrowed from~\cite{kim2008trust}.
This model assumes that customers make a decision to purchase based on their perceived risk in an online transition (e.g., defective product), the consumer's trust of a web vendor, and the perceived benefit in e-commerce~(e.g., increased convenience). Consumer trust influences perceived risk.
Here, we consider controlling customer's behavior by two kinds of advertising that correspond to adding two nodes (Ad A and Ad B) to be intervened into the model.
Ad A can change only the reliability of a website, that is, it can influence the decision of customers in an indirect way through the middle nodes.
In contrast, Ad B can change the perceived benefit.
The aim is to increase the number of purchases by consumers through choosing an effective advertisement.
This is indeed a bandit problem over a causal graph.

%The framework of the causal bandit is proposed by
The work in \cite{lattimore2016causal} considered the causal bandit problem to minimize simple regret
and offered an improved regret bound over the aforementioned tight lower-bound $\Omega(\sqrt{|\cA|/T})$~\cite{audibert2010best}[Theorem 4] for the general bandit setting~\cite{audibert2010best,gabillon2012best}.
Sen et al.~\cite{sen2017identifying} extended this study
by incorporating a smooth intervention, 
and they provided a new regret bound parameterized by the performance gap between the optimal and sub-optimal arms.
This parameterized bound comes from the technique developed for the general multi-armed bandit problem~\cite{audibert2010best}.
These analyses, however, only work for a special class of interventions with known true parameters.
Indeed, they only consider localized interventions.

\paragraph{Main contribution}
This paper proposes the first algorithm for the causal bandit problem
with an arbitrary set of interventions (which can propagate throughout the causal graph), with a theoretically guaranteed simple regret bound.
The bound is $O(\sqrt{ \gamma^*\log(|\cA|T) / T})$, where $\gamma^*$ is a parameter bounded on the basis of the graph structure.
In particular, $\gamma^* = O(N^2)$ if  the in-degree of the causal graph is bounded by a constant, 
where $N$ is the number of nodes.

The major difficulty in dealing with an arbitrary intervention comes from accumulation and propagation of estimation error.
Existing studies consider interventions that only affect the parents $\cP_k$ of a single node $V_k$.
To estimate the relationship between $\cP_k$ and $V_k$ in this setting,
we could apply an efficient importance sampling algorithm~\cite{bottou2013counterfactual, lattimore2016causal}.
On the other hand, 
when we intervene an arbitrary node,
it can affect the probabilistic propagation mechanism in any part of the causal graph.
Hence, we cannot directly control the realization of intermediate nodes
when designing efficient experiments.

The proposed algorithm consists of two steps.
First, the preprocessing step is devoted to estimating parameters for designing efficient experiments used in the main step.
More precisely, we focus on estimation of parameters with bounded {\it relative error}.
By truncating small parameters that are negligible but tend to have large relative error,
we manage to avoid accumulation of estimation error.
In the main step, we apply an importance sampling approach introduced in~\cite{lattimore2016causal,sen2017identifying}
on the basis of estimated parameters with a guaranteed relative error.
This step allows us to estimate parameters with bounded {\it absolute error},
which results in the desired regret bound.

\paragraph{Related studies}
Minimizing simple regret in bandit problems is called the best-arm identification~\cite{gabillon2012best,kaufmann2016complexity} or pure exploration~\cite{bottou2013counterfactual} problem, and it has been extensively studied in the machine learning research community.
The inference of a causal graph structure is also well-studied,
which can be classified into causal graph discovery and causal inference:
Causal graph discovery~\cite{eberhardt2005number,hauser2014two, hu2014randomized,shanmugam2015learning}
considers efficient experiments for determining the structure of causal graph,
while causal inference~\cite{mooij2016distinguishing, pearl2009causality, shimizu2011directlingam, spirtes1991algorithm} challenges one to determine the graph structure only from historical data without additional experiments.
% which requires relatively large data and several assumption.
The causal bandit problem designs experiments without using historical data, which is rather compatible with causal graph discovery studies.

\paragraph{Outline}
This paper is organized as follows.
We introduce the causal bandit problem proposed in~\cite{lattimore2016causal} in Section~\ref{sec_problem}.
We then present our bandit algorithm and regret bound in Section~\ref{sec_proposed}.
The proof of the bound is presented in Section~\ref{sec_proof}.
We offer experimental evaluation of our algorithm in Section~\ref{sec_exp}.

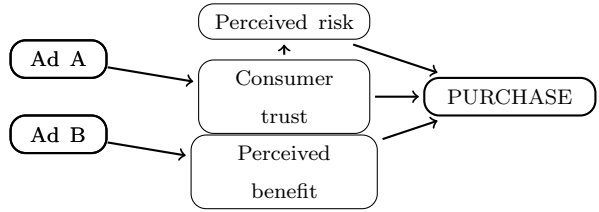
\begin{figure}\label{figure_ex_graph}
	\centering
	\begin{tikzpicture}
	\node[draw=black,thick,fill=white,outer sep=2pt,text width=2cm,text centered,rounded corners=0.2cm, minimum height=0.5cm](p) at (2,0) {\scriptsize PURCHASE};
	%\node[draw=black,text width=2.3cm,text centered,rounded corners=0.5cm](i) at (1.5,0) {Intension of purchase};
	\node[draw=black,text width=2cm,outer sep=2pt,text centered,rounded corners=0.2cm](r) at (-1, 1) {\scriptsize Perceived risk};
	\node[draw=black,text width=2cm,outer sep=2pt,text centered,rounded corners=0.2cm](c) at (-1, 0) {\scriptsize Consumer trust};
	\node[draw=black,text width=2.2cm,outer sep=2pt,text centered,rounded corners=0.2cm](b) at (-1, -1) {\scriptsize Perceived benefit};
	
	\node[draw=black,thick,text width=1cm,outer sep=2pt,text centered,rounded corners=0.2cm,minimum height=0.5cm](ada) at (-4, 0.5) {\scriptsize Ad A};
	\node[draw=black,thick,text width=1cm,outer sep=2pt,text centered,rounded corners=0.2cm,minimum height=0.5cm](adb) at (-4, -0.5) {\scriptsize Ad B};
	
	\node[draw=black,thick,text width=1cm,,text centered,rounded corners=0.2cm,minimum height=0.5cm](ada) at (-4, 0.5) {\scriptsize Ad A};
	\node[draw=black,thick,text width=1cm,text centered,rounded corners=0.2cm,minimum height=0.5cm](adb) at (-4, -0.5) {\scriptsize Ad B};
	
	\foreach \u / \v in {r/p, b/p, c/r, c/p, ada/c, adb/b}
	\draw[->,line width=0.8pt] (\u) -- (\v);
	\end{tikzpicture}
	\caption{Simple example of a causal graph.}
\end{figure}

\section{Causal bandit problem}\label{sec_problem}
This section introduces the causal bandit problem proposed by \cite{lattimore2016causal}.

Let $G=(\cV,E)$ be a directed acyclic graph (DAG)
with a node set $\cV=\{ v_1,v_2,\dots,v_N \}$
and a (directed) edge set $E$.
Let $(v_i,v_j)$ denote an edge from $v_i$ to $v_j$.
Without loss of generality,
we suppose that the nodes in $\cV$ are 
topologically sorted so that no edge from $v_i$ to $v_j$ exists if $i \geq j$.
For each $n =1,\ldots,N$, 
let $\cP_n$ denote
the index set of the parents of $v_n$, i.e., $\cP_n=\{i \in\{1,\ldots,n-1\}\colon (v_i, v_n) \in E\}$.
We then define $\overline{\cP_n} = \cP_n \cup \{n \}$.
% Moreover, let $d_n$ refer to the indegree of $v_n$
% (i.e., $d_n=|\cP_n|$).

Each node $v_n \in \cV$ is associated with a random variable $V_n$, which takes a value in $\{0,1\}$.
The distribution of $V_n$ is then influenced by the variables associated with the parents of $v_n$ (unless $V_n$ is intervened, as described below).
For each $\pi \in \{0,1\}^{\overline{\cP_n}}$,
the parameter $\alpha_n(\pi)$ defined below characterizes the distribution of $V_n$
given the realizations of its parents:
\begin{align*}
\alpha_n(\pi) := \Prob \left(V_n = \pi_n \left| 
\begin{array}{l}
V_i=\pi_i \text{ for all } i \in \cP_n, \\
\text{$v_n$ is not intervened}
\end{array}
\right) \right. .
\end{align*}
That is to say, if the parents $v_i$ for $i \in \cP_n$
are realized as $\pi_i$, then 
$V_n=\pi_n$ with probability $\alpha_n(\pi)$,
and $V_n=1-\pi_n$ with probability $1-\alpha_n(\pi)$.

Together with a DAG,
we are also given 
a set $\cA$  of interventions.
Each intervention
is identified with a vector $A \in \{*,0,1\}^N$,
where $A_n \neq *$ implies that $V_n$ is intervened and that the realization of $V_n$ is fixed as $A_n$.
Let $\pi \in \{0,1\}^{\overline{\cP}_n}$.
Given an intervention $A \in \cA$ and realizations $\pi_i$ over the parents $i \in \cP_n$,
the probability that $V_n = \pi_n$ holds is then determined as follows: 
\begin{align*}
\Prob\left(V_n = \pi_n \mid V_i = \pi_i \text{ for all } i \in \cP_n, \rdo(A)\right) = 
\begin{cases}
\alpha_n (\pi) &\text{ if $A_n =*$}, \\
1 & \text{ if $A_n =\pi_n$}, \\
0 & \text{ if $A_n = 1-\pi_n$}.
\end{cases} 
\end{align*}
This equality together with the adjacency of the causal graph $G$
completely determines the joint distribution over the variables $V_1,V_2,\dots, V_N$,
under an arbitrary intervention $A \in \cA$. 

In the causal bandit problem,
we are given a DAG $G=(\cV,E)$
and a set $\cA$ of interventions.
However, the parameters $\alpha_n$ ($n=1,\ldots,N$)
are not known.
Our ideal goal is then to find an intervention $A^* \in \cA$ that maximizes the probability $\mu (A^*)$ of realizing $V_N = 1$,
where $\mu: \cA \to [0,1]$ is defined by
\begin{align*}
\mu(A) := \Prob(V_N = 1 \mid \rdo(A)) %\label{gamma_alpha} 
\end{align*}
for each $A \in \cA$.

For this purpose, we discuss the following algorithms.
First, they estimate $\mu(A)$ ($A \in \cA$)
from $T$ experimental trials. 
Each experiment consists of the application of an intervention and 
the observation of a realization $\pi \in \{0,1\}^N$ over all nodes.
Let $\hat{\mu}(A)$ denote the estimate of $\mu(A)$.
Second, the algorithm
selects the intervention $\hat{A}$ that maximizes $\hat{\mu}$.
We evaluate the efficiency of such an algorithm with the simple regret $R_T$
defined as follows:
\begin{align*}
R_T = \mu(A^*) - E[\mu(\hat{A})].
\end{align*}
Note that, even if an algorithm is deterministic,
$\hat{A}$ includes stochasticity since the observations obtained in each experiment are produced by a stochastic process.

In this paper, we assume that $N \geq 3$ and $T \geq 2$ for ease of technical discussion.

\section{Proposed Algorithm}\label{sec_proposed}
We propose an algorithm for the causal bandit problem, 
and present regret bound of the proposed algorithm in this section.
The proofs of the bound are presented in the next section.
Let $C_n=2^{|\cP_n|}$ for each $n=1,\ldots,N$,
and $C=\sum_{n=1}^N C_n$. 
For $S \subseteq S' \subseteq [1,N]$
and $\pi \in \{0,1\}^{S'}$,
let $\pi_{S}$ denote the restriction of $\pi$ onto $S$.

\subsection{Outline of the proposed algorithm}
Recall that the purpose of the causal bandit problem
is to identify an intervention $A^*$ that maximizes $\mu(A^*)$.
This task is trivial 
if $\alpha_n$ is known for all $n=1,\ldots,N$,
because $\mu(A)$ can then be calculated 
for all $A \in \cA$.
Let $B(A)=\{\pi' \in \{0,1\}^{N} \mid  \pi'_i=A_i \text{ if } A_i \neq *, \pi'_N = 1\}$, and for $n \in [1,N]$, let $I_{n,A}$ denote
the set of nodes in $[1,n]$ which are not intervened by $A$;
$I_{n,A} := \{m \in [1,n] \mid A_m = * \}$.
$\mu(A)$ can then be represented as 
\begin{align*}
\mu(A) = \sum_{\pi \in B(A)} \prod_{m \in I_{N,A}}\alpha_{m}(\pi_{\overline{\cP_n}}).% \label{mu_alpha}
\end{align*}
Therefore, for computing $\mu$ approximately, our algorithm 
estimates $\alpha_n$ ($n=1,\ldots,N$).

In order to estimate $\alpha_n$ efficiently,
we are required to manipulate the random variables associated with the parents of $v_n$.
More concretely, to estimate $\alpha_n(\overline{\pi})$ for $\overline{\pi} \in \overline{ \cP_n}$, we require samples with realization $\omega \in \{0,1\}^N$
satisfying $\overline{\pi}_i = \omega_i$ over the parents $i \in \cP_n$ of $v_n$.
For $n=1,2,\dots,N$, $\pi \in \{0,1 \}^{\cP_n}$, and $A \in \cA$, we thus introduce the additional quantities $\beta_{n}(\pi, A)$ that denote the probability of realizing $\omega$ with $\omega_{\cP_n} = \pi$ under a given intervention $A$.
More precisely, we define
\begin{align*}
\beta_n(\pi,A)
:= 
\begin{cases}
\Prob(V_m = \pi_m, \forall m \in \cP_n \mid \rdo(A)) \quad \text{if } A_n =*,\\
0 \quad \text{otherwise.}
\end{cases} 
\end{align*}

Our algorithm consists of two phases.
The first phase estimates $\beta_n$ ($n=1,\ldots,N$),
and the second phase estimates $\alpha_n$ $(n=1,\ldots,N)$.
The algorithm requires $T/3$ experiments in the first phase,
and $2T/3$ experiments in the second phase.
In the rest of this section,
we first explain those phases
and present a regret bound on the algorithm.

\subsection{First Phase: Estimation of $\beta$}

Here, we introduce the estimation phase of $\beta_n$ for all
$n=1,\ldots,N$.
The pseudo-code of this phase is described in 
Algorithm~\ref{algo_est1}.
Algorithm~\ref{algo_est1} 
requires a positive number $\lambda$ as a parameter, which will be set to $C^3/N$.
%Here $T_1$ is the number of experiments which can be spent in this phase.
%As mentioned above, $T_1$ will be set to $T/3$,
%but here we discuss with arbitrary $T_1$ for convention.
We perform $T/3$ experiments in this phase.

Before explaining the details of Algorithm~\ref{algo_est1},
we note that $\beta_n$ can be calculated from
$\alpha_{1},\ldots,\alpha_{n-1}$.
For $\pi\in\{0,1\}^{\cP_n}$, 
let 
\begin{align}
B_n(\pi,A):= \{\pi' \in \{0,1\}^{n-1} \mid  \pi'_i=A_i \text{ if } A_i \neq * \text{ and } i \in [1,n-1], \pi'_i=\pi_i \text{ if } i \in \cP_n\} \label{df_Bn}
\end{align}
denote the set of realizations over $V_1,V_2,\dots,V_{n-1}$ that is consistent with the realization $\pi$ over $\cP_n$ and the intervention $A$.
If $A_n =*$, then $\beta_n(\pi,A)$ is then described as
\begin{align}
\beta_n(\pi, A) 
= \sum_{\pi' \in B_{n}(\pi,A) }\prod_{m \in I_{n-1,A}} \alpha_{m}(\pi'_{\overline{\cP_m}}).
\label{beta_alpha} 
\end{align}

Algorithm~\ref{algo_est1}
consists of $N$ iterations.
The $n$-th iteration computes the following objects:
\begin{itemize}
	\item an estimate $\hat{\beta}_{n}$ of $\beta_n$,
	\item $\hat{A}_{n,\pi} \in \cA$ for each $\pi \in \{0,1\}^{\cP_n}$,
	\item an estimate $\check{\alpha}_{n}$
	of $\alpha_n$, and
	\item $G_n \subseteq \{0,1\}^{\overline{\cP_n}}$.
\end{itemize}

We remark that $\check{\alpha}_n$ in Algorithm~\ref{algo_est1} are used only for computing 
an estimate $\hat{\beta}_{n}$ and are not used for estimating $\mu$.
An estimate of $\alpha_n$ is computed in the next phase of our algorithm.

At the beginning of the $n$-th iteration,
we compute $\hat{\beta}_{n}(\pi,A)$
for each $\pi \in \{0,1\}^{\cP_n}$ and $A \in \cA$
by \eqref{beta_alpha} substituting
$\check{\alpha}_{m}$ for $\alpha_m$;
\begin{align}
\hat{\beta}_{n}(\pi, A) 
= \sum_{\pi' \in B_{n}(\pi,A) }\prod_{m \in I_{n-1,A}} \check{\alpha}_{m}(\pi'_{\overline{\cP_n}}).
\label{beta_alpha_hat} 
\end{align}
Let us confirm that this $\hat{\beta}_{n}(\pi, A)$ can be computed if $\check{\alpha}_{m}$ ($m=1,\ldots,n-1$) are available.

For each $\pi \in \{0,1\}^{\cP_n}$, then,
we identify an intervention $\hat{A}_{n,\pi}$ that attains
$\max_{A \in \cA} \hat{\beta}_{n}(\pi,A)$.
Using $\hat{A}_{n,\pi}$, we compute $\check{\alpha}_{n}(\overline{\pi})$ as follows,
where $\overline{\pi}$ is an extension of $\pi$ onto $\{0,1\}^{\overline{\cP_n}}$.
%To compute $\hat{\alpha}_n(\pi')$,
We conduct
$T/(3C)$
experiments with $\hat{A}_{n,\pi}$.
Let $t_n(\pi)$ be the number of experiments 
in those 
$T/(3C)$ experiments 
in which the obtained realization $\omega \in \{0,1\}^N$ satisfies 
$\omega_i=\pi_i$ for each $i \in \cP_n$. 
Let $\overline{t_n}(\pi)$ 
be the number of experiments counted in 
$t_n(\pi)$,
where $\omega_n=1$ also holds.
We then compute $\check{\alpha}'_n(\overline{\pi})$ using the equation
\begin{equation}
\label{eq.hat_alpha_n}
\check{\alpha}'_n(\overline{\pi})=
\begin{cases}
\overline{t_n}(\pi) / t_n(\pi) & \text{ if } \overline{\pi}_n=1,\\
1- \overline{t_n}(\pi) / t_n(\pi) & \text{ if } \overline{\pi}_n=0.
\end{cases}
\end{equation}
The vector $\overline{\pi} \in \{0,1\}^{\overline{\cP_n}}$ is added to $G_n$
if 
\begin{align}
\label{eq.cond_Dn}
\check{\alpha}'_n(\overline{\pi}) \hat{\beta}_{n}(\pi,\hat{A}_{n,\pi})
\leq 2eS(\lambda),
\end{align}
where $S(\lambda)$ is defined as 
\begin{align*}
S(\lambda) := \frac{12\lambda N^2 C \log T}{T}.
\end{align*}
This $G_n$ reserves such $\overline{\pi} \in \{0,1 \}^{\overline{\cP_n}}$
that $\check{\alpha}'_{n}(\overline{\pi})$ is too small to estimate $\alpha_n(\overline{\pi})$ with sufficient accuracy.
Then $\check{\alpha}_{n}(\overline{\pi})$
is determined by replacing $\check{\alpha}'_n(\overline{\pi})$ with $0$ for $\overline{\pi} \in G_n$:
\begin{align}
\label{eq.alpha_nD}
\check{\alpha}_{n}(\overline{\pi}) := 
\begin{cases}
\check{\alpha}'_{n}(\overline{\pi}) & \text{ if } \overline{\pi} \not\in G_n,\\
0 & \text{otherwise}.
\end{cases}
\end{align}
This replacement contributes to reducing the relative estimation error of $\hat{\beta}_{n'}$ in subsequent steps ($n'=n+1,\dots,N$).

After iterating for all $n=1,2,\dots,N$,
the algorithm computes $H_n$ and $D_n$~($n=1,2,\ldots,N$) defined by 
\begin{align}
H_n &= \left\{ \overline{\pi} \in \{0,1\}^{\overline{\cP_n}} \left|  \hat{\beta}_{n}(\overline{\pi}_{\cP_n},\hat{A}_{n,\pi}) \leq 8e C^2 S(\lambda) \right.  \right\},  \label{overlineH}\\
D_n &=G_n \cup H_n.
\label{overlineD}
\end{align}
This $D_n$ contributes to bound the absolute error of the estimation of $\hat{\beta}_{n}(\overline{\pi}_{\cP_n})$ for $\overline{\pi} \not \in D_n$.
The algorithm returns an estimate $\hat{\beta}_n$ and the family $D := \{D_n \mid n=1,2,\dots,N \}$.

\begin{algorithm}[t]
	\caption{Estimation of $\beta$}\label{algo_est1}
	\begin{algorithmic}[1]
		\REQUIRE $\lambda$
		\ENSURE $\hat{\beta}_{n}$ ($n=1,\ldots,N$) and $D$
		
		\STATE $G_n \leftarrow \emptyset$ for $n =1,2,\ldots,N$
		\FOR{$n=1,2,\ldots,N$}
		\FOR{$\pi \in \{0,1\}^{\cP_n}$}
		\STATE Calculate $\hat{\beta}_{n}(\pi,A)$ for each $A \in \cA$ by \eqref{beta_alpha_hat} 
		\STATE Calculate $\hat{A}_{n, \pi} = {\argmax}_{A \in \cA} \hat{\beta}_{n}(\pi,A)$
		\STATE $t_n(\pi) \leftarrow 0$ and $\overline{t_n}(\pi)
		\leftarrow 0$
		\FOR{$j=1,\ldots,T/(3C)$}
		\STATE Conduct an experiment with $\hat{A}_{n, \pi}$ and let $\omega\in \{0,1\}^N$
		be the obtained result
		\STATE $t_n(\pi) \leftarrow t_n(\pi)+1$ if $\omega_i=\pi_{i}$ for all $i \in \cP_n$
		\STATE $\overline{t_n}(\pi) \leftarrow \overline{t_n}(\pi)+1$ 
		if $\omega_i=\pi_{i}$ for all $i \in \cP_n$ and $\omega_n=1$
		\ENDFOR
		\FOR{$k=0,1$}
		\STATE Extend $\pi$ to $\pi' \in \{0,1\}^{\overline{\cP_n}}$
		with $\pi'_n=k$
		\STATE Compute $\check{\alpha}'_{n}(\pi')$ by \eqref{eq.hat_alpha_n}
		\STATE If \eqref{eq.cond_Dn} holds, then $G_n \leftarrow G_n \cup \{\pi'\}$
		\STATE Compute $\check{\alpha}_{n}(\pi')$ by \eqref{eq.alpha_nD}
		\ENDFOR
		\ENDFOR
		\ENDFOR
		\STATE Compute $H_n$ and $D_n$ ($n=1,2,\dots,N$) by 
		\eqref{overlineH} and \eqref{overlineD}
		\STATE \textbf{return} $\hat{\beta}_{n}$ $(n=1,2,\ldots,N)$ and $D=\{D_n\mid n=1,2,\ldots,N\}$
	\end{algorithmic}
\end{algorithm}

\subsection{Second Phase: Estimation of $\alpha$}

In this phase, our algorithm computes an estimate $\hat{\alpha}_{n}$
of $\alpha_n$ for all $n=1,\ldots,N$.
The pseudo-code for this phase is given in Algorithm~\ref{algo_est2}.
As an input, it receives $\hat{\beta}_{n}$ $(n=1,\ldots,N)$ and $D$ from
Algorithm~\ref{algo_est1}.

\begin{algorithm}[t]
	\caption{Estimation of $\alpha$}\label{algo_est2}
	\begin{algorithmic}[1]
		\REQUIRE $\hat{\beta}_{n}$ ($n=1,\ldots,N$) and $D$
		\ENSURE $\hat{\alpha}_{n}$ ($n=1,\ldots,N$)
		\FOR{$n=1,2,\dots,N$ and each $\pi \in \{0,1\}^{\cP_n}$}
		\STATE $t'_n(\pi) \leftarrow 0$ and $\overline{t'_n}(\pi)\leftarrow 0$ 
		\STATE Calculate $\hat{A}_{n,\pi} := {\argmax}_{A \in
			\cA} \hat{\beta}_{n}(\pi,A)$
		\FOR{$j=1,\ldots,T/(3C)$}
		\STATE Conduct an experiment with $\hat{A}_{n,\pi}$ and let
		$\omega\in \{0,1\}^N$ be the obtained result
		\FOR{$m=1,\ldots,N$ with $(\hat{A}_{n,\pi})_m = *$}
		\STATE $t'_m(\omega_{\cP_m}) \leftarrow t'_m(\omega_{\cP_m})+1$
		\STATE $\overline{t'_m}(\omega_{\cP_m}) \leftarrow \overline{t'_m}(\omega_{\cP_m})+1$ 
		if $\omega_m=1$
		\ENDFOR 
		\ENDFOR
		\ENDFOR
		\STATE Compute an optimal solution $\hat{\eta}$ for \eqref{optProb}
		\FOR{$t=1,2,\dots,T/3$}
		\STATE Sample $A_t$ from $\cU(\hat{\eta})$
		\STATE Conduct experiment with $A_t$ and let $\omega\in
		\{0,1\}^N$ be the obtained realization
		\FOR{$n=1,\ldots,N$ with $A_n = *$}
		\STATE $t'_n(\omega_{\cP_n}) \leftarrow t'_n(\omega_{\cP_n})+1$
		\STATE $\overline{t'_n}(\omega_{\cP_n}) \leftarrow \overline{t'_n}(\omega_{\cP_n})+1$ 
		if $\omega_n=1$
		\ENDFOR
		\ENDFOR
		\FOR{$n=1,2,\dots,N$ and $\pi \in \{0,1\}^{\overline{\cP_n}}$}
		\STATE Compute $\hat{\alpha}'_{n}(\pi)$
		by \eqref{eq.check_alpha_n} and $\hat{\alpha}_{n}(\pi)$ by \eqref{eq.check_alpha_nD}
		\ENDFOR
		\STATE \textbf{return} $\hat{\alpha}_{n}$.
	\end{algorithmic}
\end{algorithm}

Algorithm~\ref{algo_est2} consists of two parts. The
first part
conducts $T/(3C)$ experiments with $\hat{A}_{n,\pi}$
(computed from $\hat{\beta}_{n}(\pi,A)$, $A \in \cA$)
for each $n=1,\ldots,N$ and $\pi \in \{0,1\}^{\cP_n}$.
This is the same process used to compute $\hat{\alpha}_{n}$
in Algorithm~\ref{algo_est1}.
Let 
\begin{align*}
D_n^{\downarrow} := \{ \pi \in \{0,1\}^{\cP_n} \mid \overline{\pi}^0, \overline{\pi}^1 \in  D_n \}
\end{align*}
where $\overline{\pi}^k$ is the extension of $\pi \in \{0,1\}^{\cP_n}$ onto $\{0,1\}^{\overline{\cP_n}}$ with $\overline{\pi}^k_n = k$.
Let us define a constant $r_{n,\pi}:= \hat{\beta}_{n}(\pi, \hat{A}_{n,\pi}) / C$
for each $n=1,\ldots,N$ and $\pi \in \{0,1\}^{\cP_n}$.
In the second part, the algorithm solves the following optimization problem:
\begin{align}
\min_{\eta \in [0,1]^{\cA}}  \max_{A \in \cA} & \sum_{n \in I_{N,A}} \sum_{ \pi \in  \{0,1\}^{\cP_n} \setminus D_n^{\downarrow}}  
\frac{\hat{\beta}_{n}^2(\pi, A)}{\sum_{A' \in \cA} \eta_{A'} \hat{\beta}_{n}(\pi,A') + r_{n,\pi}}\nonumber \\
\text{s.t. } & \sum_{A' \in \cA} \eta_{A'} = 1. 	\label{optProb}
\end{align}
Note that, for each $n=1,2,\dots,N$, $\pi \in \{0,1\}^{\cP_n} \setminus D_n^{\downarrow}$ only if
$\hat{\beta}_n(\pi, \hat{A}_{n,\pi}) > 0$ according to Line 20 of Algorithm~\ref{algo_est1}.
Thus the denominator is positive for every $\pi \in \{0,1\}^{\cP_n} \setminus D_n^{\downarrow}$,
and the above optimization problem is well-defined.
Let $\hat{\eta}$ be an optimal solution for \eqref{optProb}.
Consider the distribution $\cU(\hat{\eta})$ over $\cA$ that generates $A$ with a probability of $\hat{\eta}_A$.
For $T/3$ times, the second part samples an intervention according to $\cU(\hat{\eta})$ and
uses it to conduct experiments.

For each $n=1,\ldots,N$ and $\pi \in \{0,1\}^{\cP_n}$,
the algorithm counts the number $t'_n(\pi)$ (resp., $\overline{t'_n}(\pi)$)
of experiments that result in $\omega\in \{0,1\}^N$ with
$\omega_{\cP_n}=\pi$ (resp., $\omega_{\cP_n}=\pi$ and $\omega_n=1$).
%,  where the experiments in the first part are counted in
%$t'_n(\pi)$ and $\overline{t'_n}(\pi)$
%only when those are done in the iteration corresponding to $n$ and $\pi$.
Then, $\hat{\alpha}'_n(\pi)$ ($n =1,\ldots,N$, $\pi \in
\{0,1\}^{\overline{\cP_n}}$)
is defined by
\begin{equation}
\label{eq.check_alpha_n}
\hat{\alpha}'_n(\pi)=
\begin{cases}
\overline{t'_n}(\pi_{\cP_n}) / t'_n(\pi_{\cP_n}) & \text{ if } \pi_n=1,\\
1- \overline{t'_n}(\pi_{\cP_n}) / t'_n(\pi_{\cP_n}) & \text{ if } \pi_n=0.
\end{cases}
\end{equation}
The output $\hat{\alpha}_{n}$ defined by
\begin{align}
\label{eq.check_alpha_nD}
\hat{\alpha}_{n}(\pi)=
\begin{cases}
\hat{\alpha}'_n(\pi) & \text{if } \pi \not \in D_n,\\
0 & \text{otherwise.}
\end{cases}
\end{align}

\subsection{Regret bound}
\begin{algorithm}
	\caption{Causal Bandit}\label{algo_select}
	\begin{algorithmic}[1]
		\STATE Apply Algorithm~\ref{algo_est1} with $\lambda = C^3/N$ to obtain $\hat{\beta}_{n}$ ($n=1,\ldots,N$) and $D$
		\STATE Apply Algorithm~\ref{algo_est2} to obtain $\hat{\alpha}_{n}$ ($n=1,\ldots,N$)
		\STATE Calculate $\hat{\mu}(A)$ for each $A \in \cA$ by
		\eqref{hat_mu_alpha} 
		\STATE \textbf{return} $\hat{A} := \argmax_{A \in \cA} \hat{\mu}(A)$
	\end{algorithmic}.
\end{algorithm}

Pseudo-code of our entire algorithm is provided in
Algorithm~\ref{algo_select}.
It computes an estimate $\hat{\beta}$ of $\beta$ by Algorithm~\ref{algo_est1}
and then computes $\hat{\alpha}$
by Algorithm~\ref{algo_est2}.
It then computes an estimate $\hat{\mu}$ of $\mu$ by 
\begin{align}
\hat{\mu}(A) = \sum_{\pi \in B(A)} \prod_{n \in I_{N,A}}\hat{\alpha}_{n}(\pi_{\overline{\cP_n}}) \label{hat_mu_alpha}
\end{align}
for each $A \in \cA$. The algorithm returns an intervention $\hat{A} \in \cA$ that maximizes $\hat{\mu}$.

Let us define $\gamma^*$ as the optimum value of the following problem:
\begin{align}\label{optProb*}
\gamma^* := \min_{\eta \in [0,1]^{\cA}}  \max_{A \in \cA} & \sum_{n=1}^N \sum_{\substack{ \pi \in  \{0,1\}^{\cP_n} \\ : \beta_{n}(\pi, A) > 0 }}  \frac{\beta_{n}^2(\pi, A)}{\sum_{A' \in \cA} \eta_{A'} \beta_n(\pi,A')} \nonumber \\
\text{s.t. }& \sum_{A' \in \cA} \eta_{A'} = 1.
\end{align}
The regret bound of Algorithm~\ref{algo_select} is parameterized by
the optimum value $\gamma^*$:
\begin{thm}\label{thm_main}
	The regret $R_T$ of Algorithm~\ref{algo_select} satisfies
	\begin{align*}
	R_T \leq  O\left( \sqrt{\frac{ \max\{\gamma^*, N\} \log(|\cA|T)}{T}} \right).
	\end{align*}
\end{thm}
The notation $O(\cdot )$
is used here under the assumption that $N$ is sufficiently small with respect to $T$ but not negligible.
The optimum value $\gamma^\ast$ is bounded as follows.
Let $|A|$ denote the number of nodes intervened by $A$, i.e., $|A| := |\{ n \in [1,N] \colon A_n \in \{0,1 \} \}|$:
\begin{prop}\label{prop_gamma_bound}
	It holds that $N - \min_{A \in \cA}|A| \leq \gamma^* \leq \min\{ NC,  N|\cA| \}$.
\end{prop}
Since the lower-bound for the general best-arm identification problem
is
$\Omega(\sqrt{|\cA| / T})$ \cite{audibert2010best}[Theorem 4],
our algorithm provides a better regret bound when the number of interventions 
$|\cA|$ is large compared to $\gamma^* \leq NC$, which is only dependent on the causal graph structure.
\begin{rmk}
	We present Algorithms~\ref{algo_est1},~\ref{algo_est2},
	and~\ref{algo_select} for the setting that every $\alpha_n(\pi)$ is unknown.
	However, our algorithms can be applied even when $\alpha_n(\pi)$
	is known for some $n=1,\ldots,N$ and $\pi\in \{0,1\}^{\overline{\cP_n}}$
	by incorporating minor modifications.
	In this case,
	we denote the number of unknown $\alpha_n(\pi)$ as $C$.
	The modified algorithm just skips experiments for estimating the known $\alpha_n(\pi)$,
	and we can define $\hat{\beta}_n(\pi_{\cP_n},A)=0$ for such $n$ and $\pi$.
	We then redefine $\gamma^*$ by replacing corresponding $\beta_{n}(\pi_{\cP_n},A)$ with $0$ in \eqref{optProb*},
	and our bound in Theorem~\ref{thm_main} is valid for this reduced $\gamma^*$.
	In particular, we can recover the regret bound considered in 
	\cite{lattimore2016causal}[Theorem 3] as follows:
\end{rmk}
\begin{cor}
	Suppose that $\alpha_n(\pi)$ is known for every $n<N$ and $\pi \in \{0,1\}^{\overline{\cP_n}}$.
	Then the regret $R_T$ of Algorithm~\ref{algo_select} satisfies $R_T \leq O(\sqrt{\gamma^* \log (|\cA|T) / T })$, where 
	\begin{align*}
	\gamma^* = \min_{\eta \in [0,1]^{\cA}} & \max_{A \in \cA} \sum_{ \pi \in  \{0,1\}^{\cP_N}}  \frac{\beta_{N}^2(\pi, A)}{\sum_{A' \in \cA} \eta_{A'} \beta_N(\pi,A')}\nonumber \\
	\text{\rm s.t. } &\sum_{A' \in \cA} \eta_{A'} = 1.
	\end{align*}
\end{cor}
\begin{rmk}
	Our problem setting is often called {\it hard intervention}, which directly controls the realization of a node $v_n$ as $A_n \in \{0,1\}$.
	In contrast, Sen et al.~\cite{sen2017identifying} introduced the {\it soft intervention} model on a node $v_n$ where
	an intervention changes the conditional probability of a node $v_n$. 
	They in fact considered a simple case where a graph has a single node $v_k$ such that $\cP_N = \cP_k \cup \{k \}$, whose conditional probability can be controlled by soft intervention.
	In their model, we are given a discrete set $\cS$ as the set of soft interventions.
	For $\pi\in \{0, 1\}^{\overline{\cP_k}}$ and $S\in\cS$, define 
	\begin{align*}
	\alpha_{k}(\pi,S) :=  \Prob(V_k = \pi_k \mid V_i = \pi_i, \forall i \in \cP_k, \rdo_{\mathrm{soft}}(S) )
	\end{align*}
	as the probability of realizing $V_k = \pi_k$ under the soft intervention $\rdo_{\mathrm{soft}}(S)$ and the condition $V_i = \pi_i$ for $i \in \cP_k$.
	The goal is then to maximize the following probability: 
	\begin{align*}
	\Prob(V_N = 1 \mid \rdo_{\mathrm{soft}}(S)) = \sum_{\pi \in \{0,1 \}^{N}: \pi_N=1 } 
	\alpha_{N}(\pi_{\overline{\cP_N}}) \alpha_{k}(\pi_{\overline{\cP_k}} ,S) \cdot  \Prob (V_i = \pi_i , \forall i \in \cP_k).
	\end{align*}
	Sen et al.~\cite{sen2017identifying} proved parameterized regret bound assuming that $\Prob (V_i = \pi_i, \forall i \in \cP_k)$ is known in advance.

	We here remark that their model can be implemented by the hard intervention model as follows.
	Regard $\cS$ as the set of indices, and we add nodes $v_S$ for each $S \in \cS$ to the graph.
	Every $v_S$ has only one adjacent edge from $v_S$ to $v_k$.
	Observe that $\cP_k \cup \cS$ is the set of indices of nodes which are the parents of $v_k$ in the new graph.
	For $\pi \in \{0,1\}^{\overline{\cP_k} \cup \cS}$,
	we define $\alpha_k(\pi)$ by
	\begin{align*}
	\alpha_k(\pi) := 
	\begin{cases}
	\alpha_k(\pi_{\overline{\cP_k}}, S) \quad \text{if } \pi_S = 1 \text{ and } \pi_{S'}=0 \text{ for all the other } S' \in \cS \\
	0 \quad \text{otherwise.}
	\end{cases}
	\end{align*}
	We consider the set of hard interventions $\cA = \{A_S \mid S \in \cS \}$,
	where each intervention is indexed by $S \in \cS$,
	and $A_S$ fixes the realization of the node $v_S$ as $1$.
	More concretely,
	\begin{align*}
	A_{S, n} := 
	\begin{cases}
	1 \quad \text{if } n = S,\\
	0 \quad \text{if } n \in \cS \text{ and } n \neq S,\\
	* \quad \text{otherwise.}
	\end{cases}
	\end{align*}
	
	Then the joint distribution over the nodes $v_1,v_2,\dots,v_N$ under the soft intervention $S \in \cS$ is equal to the distribution under the corresponding hard intervention $A_S \in \cA$,
	and thus the soft intervention model is reduced to the hard intervention model.
\end{rmk}

\section{Proofs}\label{sec_proof}
This section is devoted to proving Theorem~\ref{thm_main}.
We introduce a series of well-known technical lemmas together with a novel variant of Hoeffding's inequality in Section~\ref{sec_tech}.
In Sections~\ref{subsec_algo1} and~\ref{subsec_algo2}, we ensure the accuracy of estimation in Algorithms~\ref{algo_est1} and~\ref{algo_est2}, respectively, which are presented formally as Propositions~\ref{prop_est} and~\ref{prop_bound}.
Section~\ref{subsec_algo3} then proves
Theorem~\ref{thm_main} and Proposition~\ref{prop_gamma_bound},
whose statements are presented in the previous section.

\subsection{Technical lemmas}\label{sec_tech}
We introduce Hoeffding's inequality, Chernoff's bound, and Hoeffding's lemma as follows.
\begin{prop}[Hoeffding's inequality]\label{prop_hoeffding}
	For every $i = 1,2,\dots, n$, suppose that $X_i$ is an
	independent random variable over $[a_i, b_i] \subseteq \bR$. We define $S =
	\sum_{i=1}^{n} X_i$ and $\mu = E[S]$. Then for any $\varepsilon > 0$ we have
	\begin{align*}
	\Prob\left( |S - \mu | \geq \varepsilon \right) \leq 2 \exp \left( -  \frac{2\varepsilon^2}{\sum_{i=1}^{n} (b_i - a_i)^2 } \right).
	\end{align*}
\end{prop}
\begin{prop}[Chernoff's bound]\label{prop_chernoff}
	For every $i = 1,2,\dots, n$, suppose that $X_i$ is an independent random variable over $[0, 1]$. We define $S = \sum_{i=1}^{n} X_i$ and $\mu = E[S]$.
	Then we have
	\begin{align*}
	\Prob \left( S \leq (1 - \delta)\mu   \right) \leq \exp\left( -
	\frac{\delta^2 \mu}{2} \right), \quad 0 \leq \forall \delta \leq 1,  \\
	\Prob \left( S \geq (1 + \delta)\mu   \right) \leq \exp\left( -
	\frac{\delta^2 \mu}{3} \right), \quad 0 \leq \forall \delta \leq 1,\\
	\Prob \left( S \geq (1 + \delta)\mu   \right) \leq \exp\left( -
	\frac{\delta \mu}{3} \right), \quad \forall \delta \geq 1.
	\end{align*}
\end{prop}
\begin{prop}[Hoeffding's lemma]\label{lem_hoeffdinglemma}
	Suppose that $X$ is a random variable over $[0,1]$, and define $\overline{X} = E[X]$.
	Then for any $\lambda \in \bR$ it holds that 
	\begin{align*}
	E[\exp(\lambda(X - \overline{X}))] \leq \exp\left( \frac{1}{8}\lambda^2 \right).
	\end{align*}
\end{prop}
The following statement is a variant of Hoeffding's inequality,
which is proven on the basis of Hoeffding's lemma.
\begin{lem}[Variant of Hoeffding's inequality]\label{lem_hoeffding}
	For every $m=1,2,\dots,M$ and $k\in \bN$, let $Y_{m,k}$ be a
	random variable over $\{0,1\}$.
	For each $m=1,2,\ldots,M$, we assume that
	the variables in $\mathcal{Y}_m:=\{Y_{m,k}\}_{k\in \bN}$ are
	independent and has the identical mean  $\overline{Y}_m$
	(i.e., $\overline{Y}_m = E[Y_{m,k}]$ for all $k \in \bN$)
	under the condition that the variables
	in $\bigcup_{m'=1}^{m-1} \mathcal{Y}_{m'}$ are fixed.
	For $m=1,2,\dots,M$, 
	let $T_m$ be a random variable over $\bN$ which is independent
	of $Y_{m'',k}$ if $m'' \geq m$.
	Let $\cA$ be a finite set,
	and let $\tau_m \in \bN$, $\beta_{m,A}$, and $\varepsilon_A > 0$
	be arbitrary numbers given for each $m=1,2,\dots,M$ and $A \in \cA$.
	Then we have
	\begin{align*}
	&\Prob  \left( 
	\begin{array}{r}
	\left| \sum_{m=1}^M \beta_{m,A} \left( \frac{1}{T_m} \sum_{k=1}^{T_m} Y_{m,k} - \overline{Y}_m \right) \right| \leq \varepsilon_A, 
	\forall A \in \cA  
	\end{array}
	\right) \nonumber \\
	&\geq \Prob(T_m \geq \tau_m,   \forall m=1,2,\dots,M) 
	- 2 \sum_{A \in \cA} \exp\left( - \frac{2 \varepsilon_A^2}{\sum_{m=1}^M (\beta_{m,A}^{2} / \tau_m)}  \right).
	\end{align*}
	\begin{proof}
		Let $I$ be the indicator function.
		For each $A \in \cA$,
		we introduce positive numbers $s_A, s_A' > 0$,
		which will be optimized later.
		We note that
		$\left| \sum_{m=1}^M \beta_{m,A} \left( \frac{1}{T_m}
		\sum_{k=1}^{T_m} Y_{m,k} - \overline{Y}_m \right) \right| \geq
		\varepsilon_A$ holds
		if and only if
		\begin{equation}
		s_A \sum_{m=1}^M \beta_{m,A} \left( \frac{1}{T_m}
		\sum_{k=1}^{T_m} Y_{m,k} -
		\overline{Y}_m \right)  \geq s_A
		\varepsilon_A \text{ or } 
		-s_A' \sum_{m=1}^M \beta_{m,A} \left( \frac{1}{T_m}
		\sum_{k=1}^{T_m}
		Y_{m,k} -
		\overline{Y}_m \right)
		\geq s_A' \varepsilon_A
		\label{lem15_00}
		\end{equation}
		holds. Hence,
		\begin{align}
		&\Prob  \left( 
		\begin{array}{c}
		\left| \sum_{m=1}^M \beta_{m,A} \left( \frac{1}{T_m} \sum_{k=1}^{T_m} Y_{m,k} - \overline{Y}_m \right) \right| \geq \varepsilon_A,  \exists A \in \cA 
		\end{array}
		\right) \nonumber \\
		& = \Prob \left(
		\text{\eqref{lem15_00} holds for } \exists A \in
		\cA\right) \nonumber\\ 
		& \leq
		\Prob \left(
		\text{\eqref{lem15_00} holds for } \exists A \in
		\cA \text{ and } T_{m}\geq \tau_{m},  \forall
		m=1,2,\ldots,M\right) \nonumber \\
		&\quad + 1-\Prob\left( T_{m}\geq \tau_{m}, \forall
		m=1,2,\ldots,M \right)\nonumber \\
		&=  \Prob \left(
		\begin{array}{l}
		\exp\left(s_A \sum_{m=1}^M \beta_{m,A} \left( \frac{1}{T_m} \sum_{k=1}^{T_m} Y_{m,k} - \overline{Y}_m \right)  \right)  \prod_{i=1}^M\rI(T_{i} \geq \tau_{i})\geq e^{s_A \varepsilon_A} \text{ or} \\ 
		\exp\left( -s_A' \sum_{m=1}^M \beta_{m,A} \left( \frac{1}{T_m} \sum_{k=1}^{T_m} Y_{m,k} - \overline{Y}_m \right) \right) \prod_{i=1}^M \rI(T_{i} \geq \tau_{i}) \geq e^{s_A' \varepsilon_A} ,\\
		\exists A \in \cA 
		\end{array}
		\right) \nonumber\\
		& \quad +  1 - \Prob(T_m \geq \tau_m,  \forall m =1,2,\dots,M) \nonumber\\
		&\leq \sum_{A \in \cA} \Prob \left(
		\begin{array}{l}
		\exp\left(  s_A \sum_{m=1}^M \beta_{m,A} \left( \frac{1}{T_m} \sum_{k=1}^{T_m} Y_{m,k} - \overline{Y}_m \right) \right) \prod_{i=1}^M \rI(T_{i} \geq \tau_{i}) \geq e^{s_A \varepsilon_A} 
		\end{array}
		\right) \nonumber\\ 
		&\quad + \sum_{A \in \cA} \Prob \left(
		\begin{array}{l}
		\exp\left( - s_A' \sum_{m=1}^M \beta_{m,A} \left( \frac{1}{T_m} \sum_{k=1}^{T_m} Y_{m,k} - \overline{Y}_m \right) \right) \prod_{i=1}^M \rI(T_{i} \geq \tau_{i})\geq e^{s'_A \varepsilon_A} 
		\end{array}
		\right) \nonumber\\ 
		&\quad + 1 - \Prob(T_m \geq \tau_m, \forall  m =1,2,\dots,M). \label{lem15_1}
		\end{align}
		For bounding the first term of the right-hand side, by Markov's inequality we have
		\begin{align}
		&\Prob \left(  \exp\left( s_A \sum_{m=1}^M \beta_{m,A} \left( \frac{1}{T_m} \sum_{k=1}^{T_m} Y_{m,k} - \overline{Y}_m \right)  \right) \prod_{i=1}^M \rI(T_{i} \geq \tau_{i})\geq e^{s_A \varepsilon_A}  \right)\nonumber \\
		&\leq e^{-s_A \varepsilon_A}E\left[ \exp\left( s_A \sum_{m=1}^M \beta_{m,A} \left( \frac{1}{T_m} \sum_{k=1}^{T_m} Y_{m,k} - \overline{Y}_m \right) \right) \prod_{i=1}^M \rI(T_{i} \geq \tau_{i}) \right] \label{eq_hoeff1}
		\end{align}
		for every $A \in \cA$.
		For every $A \in \cA$ and $m=1,2,\dots,M$, let us define a random variable
		\begin{align*}
		S_{A,m} &:=  \exp\left( s_A \sum_{j=1}^m \beta_{j,A} \left( \frac{1}{T_{j}} \sum_{k=1}^{T_{j}} Y_{j,k} - \overline{Y}_{j} \right) \right)\prod_{i=1}^m \rI(T_i \geq \tau_i)\\
		&= \left(\prod_{j=1}^m \prod_{k=1}^{T_j}
		\exp \left( s_A \beta_{j,A} \left( \frac{Y_{j,k} -
			\overline{Y}_{j}}{T_{j}} \right)  \right)\right) \prod_{i=1}^m  \rI(T_i \geq \tau_i) .
		\end{align*}
		Then the right-hand side of \eqref{eq_hoeff1} is equal to $e^{- s_A \varepsilon_A}E[S_{A,m}]$,
		and for each $m =1,2,\dots,M$, it holds that
		\begin{align*}
		S_{A,m} = S_{A,m-1} \rI(T_m \geq \tau_m)  \prod_{k=1}^{T_m}
		\exp \left( s_A \beta_{m,A} \left( \frac{Y_{m,k} - \overline{Y}_{m}}{T_{m}} \right)  \right).
		\end{align*}
		% 	Let $\cY_{m}$ denote the set of random variables
		% $\{Y_{m',k}\}_{1\leq m' \leq m , k=1,2,\dots}$.
		We bound $E[S_{A,m}]$ inductively, as
		\begin{align*}
		E\left[ S_{A,m} \right] 
		&= E_{\cY_1,\dots,\cY_{m-1}, T_m}\left[ S_{A,m-1}  \rI(T_m \geq \tau_m) E_{ \cY_m }\left[\left.  \prod_{k =1}^{T_m}  \exp\left( s_A \beta_{m,A} \frac{Y_{m,k} - \overline{Y}_m}{T_m} \right) \right| \cY_1,\dots,\cY_{m-1}, T_m     \right] \right]\\
		&= E_{\cY_1,\dots,\cY_{m-1}, T_m}\left[ S_{A,m-1}  \rI(T_m \geq \tau_m) \prod_{k =1}^{T_m} E_{ Y_{m,k}}\left[\left.  \exp\left( s_A \beta_{m,A} \frac{Y_{m,k} - \overline{Y}_m}{T_m} \right) \right| \cY_1,\dots,\cY_{m-1} , T_m  \right] \right]\\
		&\leq E_{ \cY_1,\dots,\cY_{m-1}, T_m }\left[ \rI(T_m \geq \tau_m) S_{A,m-1} \prod_{k =1}^{T_m} \exp\left( \frac{ s_A^2 \beta_{m,A}^{2} }{8T_m^2}  \right)  \right] \\
		&\leq E_{ \cY_1,\dots,\cY_{m-1}, T_m }\left[ \rI(T_m \geq \tau_m) S_{A,m-1} \exp\left( \frac{ s_A^2 \beta_{m,A}^{2} }{8T_m}  \right)  \right] \\
		&\leq  E\left[ S_{A,m-1}   \right] \exp\left( \frac{ s_A^2 \beta_{m,A}^{2} }{8 \tau_m} \right),
		\end{align*}
		where the first inequality follows from Proposition~\ref{lem_hoeffdinglemma} by putting $\lambda = s_A \beta_{m,A} / T_m$.
		Thus, by induction, we have  
		\begin{align*}
		E[S_{A,m}] \leq \exp\left( s_A^2 \sum_{m=1}^M \frac{ \beta_{m,A}^{2}}{8\tau_m } \right),
		\end{align*}
		which implies
		\begin{align*}
		\text{\eqref{eq_hoeff1}} \leq \exp\left(-s_A \varepsilon_A + s_A^2\sum_{m=1}^M \frac{  \beta_{m,A}^{2} }{8\tau_m} \right).
		\end{align*}
		Putting 
		\begin{align*}
		s_A = \frac{4\varepsilon_A}{\sum_{m=1}^M (\beta_{m,A}^{2} / \tau_m)},
		\end{align*}
		for each $A \in \cA$, we have
		\begin{align*}
		\text{\eqref{eq_hoeff1}} \leq \exp\left( - \frac{2\varepsilon_A^2}{\sum_{m=1}^{M} (\beta_{m,A}^{2} / \tau_m)}   \right).
		\end{align*}
		Thus the first term of \eqref{lem15_1} is bounded above by
		\begin{align} 
		\sum_{A \in \cA}\exp\left( - \frac{2\varepsilon_A^2}{\sum_{m=1}^{M} (\beta_{m,A}^{2} / \tau_m)}   \right). \label{Lem15_2}
		\end{align}
		
		We can apply the same technique for bounding the second term of \eqref{lem15_1}.
		In short, Markov's inequality, the induction, and Lemma~\ref{lem_hoeffdinglemma} with $\lambda = - s'_A \beta_{m,A} / T_m$ show
		\begin{align*}
		&\Prob \left( \prod_{m'=1}^M \rI(T_{m'} \geq \tau_{m'}) \exp\left( - s'_A \sum_{m=1}^M \beta_{m,A} \left( \frac{1}{T_m} \sum_{k=1}^{T_m} Y_{m,k} - \overline{Y}_m \right)  \right) \geq e^{s'_A \varepsilon_A}  \right) \\
		&\leq \exp\left(- s'_A \varepsilon_A + s_A^{'2}\sum_{m=1}^M \frac{  \beta_{m,A}^{2} }{8\tau_m} \right)
		\end{align*}
		for each $A$. By putting 
		\begin{align*}
		s_k' = \frac{\varepsilon_k}{\sum_{i=1}^n (\beta_{i,k}^{2} / 4\tau_i)},
		\end{align*}
		then, the second term of \eqref{lem15_1} is also bounded above by \eqref{Lem15_2}.
		Therefore we have 
		\begin{align*}
		&\Prob  \left( 
		\begin{array}{c}
		\left| \sum_{m=1}^M \beta_{m,A} \left( \frac{1}{T_m} \sum_{k=1}^{T_m} Y_{m,k} - \overline{Y}_m \right) \right| \geq \varepsilon_A,  \exists A \in \cA 
		\end{array}
		\right) \\
		&\leq 2\sum_{A \in \cA}\exp\left( -
		\frac{2\varepsilon_A^2}{\sum_{m=1}^{M} (\beta_{m,A}^{2}/ \tau_m)}   \right)  + 1 - \Prob(T_m \geq \tau_m,
		\forall m =1,2,\dots,M),
		\end{align*}
		which is equivalent to the desired bound.
	\end{proof}
\end{lem}

\subsection{Accuracy of Algorithm~\ref{algo_est1}}\label{subsec_algo1}
For $n=1,2,\dots,N$, 
let $\check{\alpha}_n$ and $\check{\alpha}'_n$
be the stochastic estimates computed in Algorithm~\ref{algo_est1},
and $\hat{A}_{n,\pi} := \argmax_{A \in \cA} \hat{\beta}_n(\pi,A)$
be the action determined from the estimate $\hat{\beta}_n$.

Let $G$ be defined by $G=\{G_n\mid n = 1,\dots, N\}$.
Using $G$, we define $\alpha_{n,G}$ and $\beta_{n,G}$ as follows.
For each $n \in [1,N]$ and $\overline{\pi} \in \{0,1\}^{\overline{ \cP_n}}$,
we define $\alpha_{n,G}(\overline{\pi})$ by
\begin{align*}
\alpha_{n,G}(\overline{\pi}) &:= 
\begin{cases}
\alpha_{n}(\overline{\pi}) & \text{ if } \overline{\pi} \not\in G_n,\\
0 & \text{otherwise}.
\end{cases}
\end{align*}
For each $n \in [1,N]$, $\pi \in \{0,1\}^{\cP_n}$,
and $A\in \cA$,
we define $\beta_{n,G}(\pi,A)$ by
\begin{align}
\beta_{n,G}(\pi, A) 
& = \sum_{\pi' \in B_{n}(\pi,A) }\prod_{m \in I_{n-1,A}} \alpha_{m,G}(\pi'_{\overline{\cP}_n}).
\label{betaD_alpha}
\end{align}
Thus $\alpha_{n,G}(\overline{\pi})$ is obtained from $\alpha_{n}(\overline{\pi})$ by truncating its values if $\overline{\pi}\in G_n$,
and $\beta_{n,G}$ is defined from $\alpha_{n,G}$.
We define $\alpha_{n,D}$ and $\beta_{n,D}$ in the same way.
Since $G_n \subseteq D_n$,
we observe that $\beta_{n,D}(\pi, A) \leq \beta_{n,G}(\pi, A) \leq \beta_n(\pi, A) $.
Similarly, for $A\in \cA$, we define $\mu_{D}(A)$ by
\begin{align}
\mu_{D}(A) &= \sum_{\pi \in B(A)} \prod_{m \in I_{N,A}}\alpha_{m,D}(\pi_{\overline{ \cP_m}}). \label{muD}
\end{align}

The following proposition demonstrates the error bound for outputs $\hat{\beta}_{n}$ and $D$ from Algorithm~\ref{algo_est1}.
\begin{prop}\label{prop_est}
	Let $\hat{\beta}_{n}$ and $D$ be the outputs of Algorithm~\ref{algo_est1} with parameter $\lambda \geq 1$.
	Then the following holds with a probability of at least $1 - 6C/T$: 
	for every $n \in [1,N]$, $\overline{\pi} \in \{0,1\}^{\overline{\cP_n}} \setminus D_n$ with $\pi = \overline{\pi}_{\cP_n}$, and $A \in \cA$:
	\begin{align}
	&\frac{1}{e} \beta_{n,D}(\pi, A) \leq \hat{\beta}_{n}(\pi, A) \leq  e \beta_{n}(\pi, A), \label{beta_ratio1} \\
	&\alpha_n(\overline{\pi}) \beta_{n}(\pi, \hat{A}_{n,\pi}) \geq  S(\lambda)\label{alphabeta1},\\
	&\beta_{n}(\pi,A)  \leq e \hat{\beta}_{n}(\pi,A) + e \hat{\beta}_{n}(\pi, \hat{A}_{n,\pi})/C \label{beta_diff} \\
	&\mu(A) - \mu_{D} (A) \leq 8e^2(C^3 + C)S(\lambda). \label{muD_diff}
	\end{align}
\end{prop}
We prepare the following three lemmas to prove Proposition~\ref{prop_est}.
Let $\rI( \cdot )$ be the indicator function.
The first lemma is an application of Chernoff's bound,
which bounds the relative estimation error on $\check{\alpha}'_n$:
\begin{lem}\label{lem_bound1}
	Let $n \in [1,N]$, $\overline{\pi} \in \{0,1\}^{\overline{\cP_n}}$, and $\pi = \overline{\pi}_{\cP_n}$.
	
	(i) If $\alpha_n(\overline{\pi}) \beta_{n}(\pi,\hat{A}_{n,\pi}) \leq S(\lambda)$,
	then the following holds with a probability of at least $1 - 2/T$:
	\begin{align*}
	\check{\alpha}'_n(\overline{\pi}) \beta_{n}(\pi,\hat{A}_{n,\pi}) \leq 2S(\lambda).
	\end{align*}
	
	(ii) If $\alpha_n(\overline{\pi}) \beta_{n}(\pi,\hat{A}_{n,\pi}) \geq S(\lambda)$,
	then the following holds with a probability of at least $1 - 3/T$:
	\begin{align*}
	\left( 1 - \frac{1}{\sqrt{\lambda}N} \right) \alpha_{n} (\overline{\pi}) 
	\leq \check{\alpha}'_{n} (\overline{\pi}) 
	\leq  \left( 1 + \frac{1}{\sqrt{\lambda}N} \right) \alpha_{n} (\overline{\pi}) 
	\end{align*}
	\begin{proof}
		The cases $\overline{\pi}_n =1$ and $\overline{\pi}_n = 0$ are symmetric,
		and thus without loss of generality we can assume that $\overline{\pi}_n =1$.
		If $(\hat{A}_{n,\pi})_n \neq *$, then $\beta_{n}(\pi,\hat{A}_{n,\pi}) = 0$ for any $\pi$ and thus (i) immediately holds.
		Therefore the following discussion assumes that $(\hat{A}_{n,\pi})_n = *$.
		Let $\omega^{(1)},\omega^{(2)},\dots,\omega^{(T/(3C))} \in \{0,1\}^N$ be the obtained realization
		over $T/(3C)$ experiments in Algorithm~\ref{algo_est1} with $\hat{A}_{n,\pi}$.
		Then we define i.i.d. random variables $X_{t}$ for $t=1,2,\dots,T/(3C)$ by
		\begin{align*}
		X_{t} = 
		\begin{cases}
		1 \quad \text{ if }  \pi_i = \omega^{(t)}_{i} \text{ for } \forall i \in \cP_n,\\
		0 \quad \text{ otherwise.}
		\end{cases}
		\end{align*}
		Observe that $E[X_t] = \beta_{n}(\pi,\hat{A}_{n,\pi})$ for $t=1,2,\dots,T/(3C)$.
		We also define a random variable $T_n(\pi)$ by
		\begin{align*}
		T_n(\pi) = \sum_{t=1}^{T/(3C)}X_t.
		\end{align*}
		Then it holds that 
		\begin{align}
		E[T_n(\pi)] = \frac{ \beta_{n}(\pi,\hat{A}_{n,\pi}) T}{3C}. \label{num_sample}
		\end{align}
		
		Given $X_t$ for $t=1,2,\dots, T/(3C)$ and $T_n(\pi)$, we define i.i.d. random variables $Y_k$ for $k=1,2,\dots, T_n(\pi)$ as follows:
		Let $t_1,t_2,\ldots,t_{T_n(\pi)}$ be the 
		indices such that $1 \leq t_1 < t_2 < \cdots t_{T_n(\pi)} \leq T/(3C)$
		and $X_{t_k} = 1$ for $k=1,2,\dots, T_n(\pi)$.
		Then, for each $k=1,2,\dots, T_n(\pi)$, a variable $Y_k$ is defined by
		\begin{align*}
		Y_{k} = 
		\begin{cases}
		1 \quad \text{ if } \omega^{(t_k)}_{n} = 1, \\
		0 \quad \text{ otherwise.}
		\end{cases}
		\end{align*}
		Observe that $E[Y_k \mid X_1,X_2,\dots,X_{T/(3C)}, T_n(\pi)] = E[Y_k \mid T_n(\pi)] = \alpha_n(\overline{\pi})$.
		In addition, we define a random variable $\overline{T_n}$ by
		\begin{align*}
		\overline{T_n}(\pi) =  \sum_{k=1}^{T_n(\pi)} Y_k.
		\end{align*}
		Then it holds that
		\begin{align}
		E[\overline{T_n}(\pi) \mid T_n(\pi)] = \alpha_n(\overline{\pi}) T_n(\pi),\label{ETT} 
		\end{align}
		and 
		\begin{align}
		\check{\alpha}'_n(\overline{\pi}) = \overline{T_n}(\pi) / T_n(\pi). \label{alphaTT}
		\end{align}

		(i) Let us define
		\begin{align*}
		\delta = \frac{2 S(\lambda)}{\alpha_n(\overline{\pi}) \beta_n(\pi,\hat{A}_{n,\pi})} -1.
		\end{align*}
		Since $\alpha_n(\overline{\pi}) \beta_n(\pi,\hat{A}_{n,\pi}) \leq S(\lambda)$, we have
		\begin{align}
		\delta \geq \frac{S(\lambda)}{\alpha_n(\overline{\pi}) \beta_n(\pi,\hat{A}_{n,\pi})} \geq 1. \label{delta_ineq}
		\end{align}
		
		Below we prove that
		\begin{align}
		& \Prob\left(\check{\alpha}'_n(\overline{\pi})
		\beta_{n}(\pi,\hat{A}_{n,\pi}) \leq 2S(\lambda)
		\right) \nonumber \\
		&= E_{T_n(\pi)} \left. \left[\Prob \left(
		\overline{T_n}(\pi) \leq (1+\delta)
		E[\overline{T_n}(\pi) \mid T_n(\pi)]  \right| T_n(\pi)
		\right)\right] \label{eq.i-1}\\
		&\geq E_{T_n(\pi)}\left[ 1 - \exp \left(
		-\frac{\alpha_n(\overline{\pi}) T_n(\pi) \delta}{3}
		\right) \right]\label{eq.i-2} \\
		&\geq 1 - \frac{2}{T}, \label{eq.i-3}
		\end{align}
		where the details of each transformation will be explained as follows.
		
		\eqref{eq.i-1} follows from the following equivalence:
		\begin{align*}
		\check{\alpha}'_n(\overline{\pi}) \beta_n(\pi,\hat{A}_{n,\pi}) \leq 2S(\lambda) 
		&\Leftrightarrow \frac{\overline{T_n}(\pi)}{T_n(\pi)}\beta_n(\pi,\hat{A}_{n,\pi})  \leq 2 S(\lambda)\\
		&\Leftrightarrow \overline{T_n}(\pi) \leq  \frac{2S(\lambda)}{\alpha_n(\overline{\pi})\beta_n(\pi,\hat{A}_{n,\pi})} \cdot  \alpha_n(\overline{\pi})T_n(\pi)\\
		&\Leftrightarrow \overline{T_n}(\pi) \leq (1 + \delta) E[ \overline{T_n}(\pi) \mid T_n(\pi)].
		\end{align*}
		Here, the first equivalence follows from \eqref{alphaTT}, and
		the last equivalence follows from \eqref{ETT} and the
		definition of $\delta$.
		
		For \eqref{eq.i-2}, applying Proposition~\ref{prop_chernoff} to $Y_k$ given $T_n(\pi)$ with $\delta \geq 1$, we have
		\begin{align*}
		\Prob \left. \left( \overline{T_n}(\pi) \leq (1+\delta) E[\overline{T_n}(\pi) \mid T_n(\pi)]  \right| T_n(\pi) \right) &\geq 1 - \exp\left( -\frac{E[\overline{T_n}(\pi) \mid T_n(\pi)] \delta }{3}  \right)\\
		&= 1 - \exp\left( -\frac{\alpha_n(\overline{\pi}) T_n(\pi) \delta }{3}  \right)
		\end{align*}
		
		For \eqref{eq.i-3}, we first observe that $S(\lambda) \geq 24 C \log T / T$ since $\lambda\geq 1$ and $N \geq 3$. 
		If $\beta_{n}(\pi,\hat{A}_{n,\pi}) \leq 24C \log T / T \leq S(\lambda)$, then Lemma~\ref{lem_bound1}~(i) is trivial. Hence
		we may assume that $\beta_{n}(\pi,\hat{A}_{n,\pi}) \geq 24C \log T/T$.
		Then,
		\begin{align}
		E[T_n(\pi)] = \beta_{n}(\pi,\hat{A}_{n,\pi})T / (3C) \geq 8 \log T \label{eTn}.
		\end{align}
		Observe that 
		\begin{align*}
		\Prob\left( T_n(\pi) \geq \frac{E[T_n(\pi)]}{2}\right) \geq 1 - \exp\left(-\frac{E[T_n(\pi)]}{8}\right) \geq 1 - \frac{1}{T}.
		\end{align*}
		Here, the first inequailty is obtained by applying
		Proposition~\ref{prop_chernoff} to $T_n = \sum_{t=1}^{T/(3C)}X_t$
		with $\delta=1/2$, and the second inequality follows from  \eqref{eTn}.
		Then we have
		\begin{align*}
		&E_{T_n(\pi)}\left[ 1 - \exp \left(- \frac{\alpha_n(\overline{\pi}) T_n(\pi) \delta}{3} \right) \right] \\
		&\geq E_{T_n(\pi)}\left[ \rI\left( T_n(\pi) \geq \frac{E[T_n(\pi)] }{2} \right)  \left( 1 - \exp \left(- \frac{\alpha_n(\overline{\pi}) T_n(\pi) \delta}{3} \right) \right) \right] \\
		&\geq \Prob\left(T_n(\pi) \geq \frac{E[T_n(\pi)] }{2}\right) 
		\left(1 - \exp \left(- \frac{\alpha_n(\overline{\pi}) \delta}{3} \cdot
		\frac{E[T_n(\pi)] }{2} \right) \right)\\
		&\geq \left(1 - \frac{1}{T} \right) 
		\left(1 - \exp \left(-
		\frac{\alpha_n(\overline{\pi}) \delta E[T_n(\pi)]}{6} \right) \right)
		\end{align*}
		Since $E[T_n(\pi)] = \beta_{n}(\pi,\hat{A}_{n,\pi}) T / (3C)$ and $\delta
		\geq S(\lambda) / \alpha_n(\overline{\pi})
		\beta_n(\pi,\hat{A}_{n,\pi})$ by \eqref{delta_ineq},
		the right-hand side of this inequality is at least
		\begin{align*}
		&\left(1 - \frac{1}{T} \right) \left(1 - \exp \left( -\frac{\alpha_n(\overline{\pi}) }{6} \cdot \frac{S(\lambda)}{\alpha_n(\overline{\pi}) \beta_n(\pi,\hat{A}_{n,\pi})} \cdot \frac{T\beta_n(\pi,\hat{A}_{n,\pi}) }{3C} \right) \right)\\
		&= \left(1 - \frac{1}{T}\right) \left( 1 - \exp \left(- \frac{S(\lambda)T}{18C} \right) \right) \\
		&= \left(1 - \frac{1}{T}\right) \left( 1 - \exp \left(- \frac{2 \lambda N^2 \log T}{3} \right) \right)\\
		&\geq \left(1 - \frac{1}{T}\right) \left(1 - \frac{1}{T}\right)\\
		&\geq 1 - \frac{2}{T}.
		\end{align*}
		The first inequality holds since $\lambda \geq 1$ and $N \geq 3$.
		This completes the proof of (i).
		
		(ii) Putting $\delta = 1 / \sqrt{\lambda}N$, we prove that
		\begin{align}
		&\Prob\left(
		\left( 1 - \delta \right) \alpha_{n} (\overline{\pi}) 
		\leq \check{\alpha}'_{n} (\overline{\pi}) 
		\leq  \left( 1 + \delta \right) \alpha_{n} (\overline{\pi}) \right)\\
		&= E_{T_n(\pi)} \left[ \Prob\left( \left.
		\left( 1 - \delta \right) E[\overline{T_n}(\pi) \mid T_n(\pi)]
		\leq \overline{T_n}(\pi)
		\leq  \left( 1 + \delta \right) E[\overline{T_n}(\pi) \mid T_n(\pi)] \right| T_n \right) \right] \label{below(ii)-1}  \\
		&\geq E_{T_n(\pi)} \left[ 1 - 2\exp\left( - \frac{\alpha_{n} (\overline{\pi})  T_n(\pi)\delta^2}{3} \right) \right] \label{below(ii)-2} \\
		&\geq 1 - \frac{3}{T}, \label{below(ii)-3}
		\end{align}
		where the details of each transformation will be explained as follows.
		
		The first equality \eqref{below(ii)-1} holds since
		\begin{align*}
		&\left( 1 - \delta \right) \alpha_{n} (\overline{\pi}) 
		\leq \check{\alpha}'_{n} (\overline{\pi}) 
		\leq  \left( 1 + \delta \right) \alpha_{n} (\overline{\pi}) \\
		&\Leftrightarrow \left( 1 - \delta \right) \alpha_{n} (\overline{\pi}) T_n(\pi)
		\leq \overline{T_n}(\pi)
		\leq  \left( 1 + \delta \right) \alpha_{n} (\overline{\pi})T_n(\pi) \\
		&\Leftrightarrow \left( 1 - \delta \right) E[\overline{T_n}(\pi) \mid T_n(\pi)]
		\leq \overline{T_n}(\pi)
		\leq  \left( 1 + \delta \right) E[\overline{T_n}(\pi) \mid T_n(\pi)].
		\end{align*}
		The first equivalence follows from \eqref{alphaTT}, and the last equivalence follows from \eqref{ETT}.
		
		For the first inequality \eqref{below(ii)-2}, applying  Proposition~\ref{prop_chernoff} to $Y_k$ given $T_n(\pi)$ and $\delta$, we have
		\begin{align*}
		&\Prob\left( \left.
		\left( 1 - \delta \right) E[\overline{T_n}(\pi) \mid T_n(\pi)]
		\leq \overline{T_n}(\pi)
		\leq  \left( 1 + \delta \right) E[\overline{T_n}(\pi) \mid T_n(\pi)] \right| T_n \right) \\
		&\leq 1 - 2 \exp\left( -  \frac{E[\overline{T_n}(\pi) \mid T_n(\pi)] \delta^2}{3} \right)\\
		&\leq 1 - 2 \exp\left( -  \frac{ \alpha_{n} (\overline{\pi}) T_n(\pi) \delta^2}{3} \right)
		\end{align*}
		
		For the second inequality \eqref{below(ii)-3}, applying Proposition~\ref{prop_chernoff} to $T_n = \sum_{t=1}^{T/(3C)}X_t$ with $\delta'=1/4$,
		we have
		\begin{align*}
		\Prob\left(T_n(\pi) \geq (1-\delta') E[T_n(\pi)]\right) &\geq 1 - \exp \left(- \frac{\delta^{'2} E[T_n(\pi)]}{2} \right) = 1 - \exp \left(- \frac{ E[T_n(\pi)]}{32} \right) \\
		&\geq 1 - \frac{1}{T},
		\end{align*}
		where the second inequality holds since 
		\begin{align*}
		E[T_n(\pi)] = \frac{\beta_n(\pi,\hat{A}_{n,\pi})T}{3C} \geq \frac{\alpha_{n} (\overline{\pi}) \beta_n(\pi,\hat{A}_{n,\pi})T}{3C} \geq \frac{S(\lambda)T}{3C} \geq 4\lambda N^2 \log T \geq 32 \log T,
		\end{align*}
		as $N\geq3$ and $\lambda \geq 1$.
		Then we have
		\begin{align*}
		&E_{T_n(\pi)} \left[ 1 - 2\exp\left( - \frac{\alpha_{n} (\overline{\pi})  T_n(\pi)\delta^2}{3} \right) \right]\\
		&\geq E_{T_n(\pi)} \left[ \rI\left( T_n(\pi) \geq \frac{3}{4} E[T_n(\pi)] \right) \left(1 - 2\exp\left( - \frac{\alpha_{n} (\overline{\pi})  T_n(\pi)\delta^2}{3} \right) \right) \right] \\
		&\geq \Prob\left( T_n(\pi) \geq \frac{3}{4} E[T_n(\pi)] \right) \left( 1 - 2\exp\left( - \frac{\alpha_{n} (\overline{\pi})  \delta^2}{3} \cdot \frac{3E[T_n(\pi)]}{4} \right)  \right)\\
		&\geq \left(1 - \frac{1}{T} \right) \left( 1 - 2\exp\left( - \frac{\alpha_{n} (\overline{\pi})  \delta^2}{3} \cdot \frac{3E[T_n(\pi)]}{4} \right)  \right) \\
		&= \left(1 - \frac{1}{T} \right) \left( 1 - 2\exp\left( - \frac{\alpha_{n} (\overline{\pi}) }{3\lambda N^2} \cdot \frac{ \beta_n(\pi,\hat{A}_{n,\pi}) T}{4C} \right)  \right) \\
		&\geq \left(1 - \frac{1}{T} \right) \left( 1 - 2\exp\left( - \frac{S(\lambda) T }{12\lambda N^2C} \right)  \right)\\
		&= \left(1 - \frac{1}{T} \right)\left(1 - \frac{2}{T} \right)\\
		&\geq 1 - \frac{3}{T}.
		\end{align*}
		The first equality holds by the definition of $\delta$ and $E[T_n] = \beta_n(\pi,\hat{A}_{n,\pi}) T / (3C)$, and the forth inequality holds since $ \alpha_n(\overline{\pi}) \beta_n(\pi,\hat{A}_{n,\pi}) \geq S(\lambda )$.
		The proof is complete.
	\end{proof}
\end{lem}
The second lemma bounds the gap produced by truncation of $\alpha_{n}$ that is conducted for introducing
$\alpha_{n,G}$ and $\alpha_{n,D}$. 
We use the notation $H_n^{\downarrow} := \{ \pi_{\cP_n} \mid \pi \in H_n \}$.
\begin{lem}\label{lem_delete}
	(i) Let $n \in [1,N]$ and $\pi \in \{0,1\}^{\cP_n}$. For every $A \in \cA$, it holds that 
	\begin{align*}
	\beta_{n}(\pi,A) - \beta_{n,G}(\pi,A) 
	\leq  \sum_{m=1}^{N} \sum_{\pi' \in G_m} \max_{A' \in \cA} \alpha_m(\pi')\beta_{m,G}(\pi'_{\cP_m}, A').
	\end{align*}
	
	(ii) For every $A \in \cA$, it holds that 
	\begin{align*}
	\mu(A) - \mu_{D} (A) 
	\leq \sum_{m=1}^N\sum_{\pi \in G_m}  \max_{A' \in \cA} \alpha_m(\pi)\beta_{m,G}(\pi_{\cP_m}, A')
	\quad +  \sum_{m=1}^N \sum_{\pi' \in H^{\downarrow}_m} \max_{A'' \in \cA} \beta_{m,G}(\pi',A'').
	\end{align*}
	\begin{proof}
		For $m=0,1,\dots,N$, let $G^m$ be defined by
		\begin{align*}
		G^m := \left\{ G^m_{1}, G^m_{2},\dots,G^m_{N} \left| G^m_i = 
		\begin{cases}
		G_i \quad 1 \leq i \leq m \\
		\emptyset \quad i \geq m+1
		\end{cases}  
		\right\} \right.
		\end{align*}
		We define $\alpha_{n,G^m}$ and $\beta_{n,G^m}$ by replacing $G$ by $G^m$ in the definition of $\alpha_{n,G^m}$ and $\beta_{n,G}$, respectively.
		By definition, we see $\beta_{n,G^0}=\beta_n$ and $\beta_{n,G^{n-1}}=\beta_{n, G}$.
		Then it holds that
		\begin{align}
		\beta_{n}(\pi,A) - \beta_{n,G}(\pi,A) = \sum_{m=1}^{n-1} (\beta_{n,G^{m-1}}(\pi,A) - \beta_{n,G^{m}}(\pi,A)). \label{betaGsum}
		\end{align}
		
		(i) We prove that 
		\begin{align}
		\beta_{n,G^{m-1}}(\pi,A) - \beta_{n,G^{m}}(\pi,A) \leq \max_{A' \in \cA} \alpha_{m}(\pi') \beta_{m,G}(\pi'_{\cP_m},A') \label{betaSubBeta}
		\end{align}
		for every $m=1,2,\dots,n-1$. This and \eqref{betaGsum} directly imply the desired bound (i).
		For $n=1,2,\dots,N$ and $\pi \in \{0,1\}^{\overline{ \cP_{n}}}$, observe that $\alpha_{n,G^{m-1}}$ and $\alpha_{n,G^{m}}$ differ only when
		$n = m$ and $\pi \in G_m$, which implies that $\alpha_{m,G^{m}}(\pi) = 0$ and $\alpha_{m,G^{m-1}}(\pi) = \alpha_{m}(\pi)$. 
		
		If $m \not \in I_{n-1,A}$, then the expansion~\eqref{beta_alpha} implies
		\begin{align*}
		\beta_{n,G^{m-1}}(\pi,A) - \beta_{n,G^{m}}(\pi,A) &= \sum_{\pi' \in B_{n}(\pi,A) } \left( \prod_{i \in I_{n-1,A}} \alpha_{i,G^{m-1}}(\pi'_{\overline{\cP_i}}) - \prod_{j \in I_{n-1,A}} \alpha_{j,G^m}(\pi'_{\overline{\cP_j}}) \right)\\
		&=  \sum_{\pi' \in B_{n}(\pi,A) } \left( \prod_{i \in I_{n-1,A}} \alpha_{i,G^{m}}(\pi'_{\overline{\cP_i}}) - \prod_{j \in I_{n-1,A}} \alpha_{j,G^m}(\pi'_{\overline{\cP_j}}) \right) = 0.
		\end{align*}
		Thus \eqref{betaSubBeta} holds.
		
		If $m \in I_{n-1, A}$, then it holds that
		\begin{align*}
		\beta_{n,G^{m-1}}(\pi,A) - \beta_{n,G^{m}}(\pi,A) &= \sum_{\pi' \in B_{n}(\pi,A) } \left( \prod_{i \in I_{n-1,A}} \alpha_{i,G^{m-1}}(\pi'_{\overline{\cP_i}}) - \prod_{j \in I_{n-1,A}} \alpha_{j,G^m}(\pi'_{\overline{\cP_j}}) \right)\\
		&= \sum_{\pi' \in B_{n}(\pi,A) } \left( \alpha_{m,G^{m-1}}(\pi'_{\overline{\cP_m}}) - \alpha_{m,G^{m}}(\pi'_{\overline{\cP_m}})\right) \prod_{i \in I_{n-1,A}: i\neq m} \alpha_{i,G^{m-1}}(\pi'_{\overline{\cP_i}})\\
		&= \sum_{\pi' \in B_{n}(\pi,A) : \pi'_{\overline{\cP_m}} \in G_m  } \alpha_{m}(\pi'_{\overline{\cP_m}}) \prod_{i \in I_{n-1,A}: i\neq m} \alpha_{i,G^{m-1}}(\pi'_{\overline{\cP_i}})\\
		&= \sum_{\pi' \in G_m} \alpha_{m}(\pi') \sum_{ \substack{\pi'' \in B_n(\pi, A)\\ : \pi''_i = \pi'_i , i \in \overline{ \cP_m} }} \prod_{i \in I_{n-1,A}: i \leq m-1} \alpha_{i,G}(\pi'_{\overline{\cP_i}}) \prod_{j \in I_{n-1,A} : j \geq m+1} \alpha_{j}(\pi''_{\overline{\cP_j}} ).
		\end{align*}
		Suppose that the following holds for each $\pi' \in \{0,1\}^{\overline{\cP_m}}$:
		\begin{align}
		\sum_{ \substack{\pi'' \in B_n(\pi, A)\\ : \pi''_i = \pi'_i , i \in \overline{ \cP_m} }} 
		\prod_{i \in I_{n-1,A}: i \leq m-1}\alpha_{i,G}(\pi''_{\overline{\cP_i}} )
		\prod_{j \in I_{n-1,A} : j \geq m+1} \alpha_{j}(\pi''_{\overline{\cP_j}} )
		\leq  \beta_{m,G}(\pi'_{\cP_m},A). \label{Galphabeta}
		\end{align}
		Then we have \eqref{betaSubBeta} as follows:
		\begin{align*}
		\beta_{n,G^{m-1}}(\pi,A) - \beta_{n,G^{m}}(\pi,A) 	
		= \sum_{\pi' \in G_m} \alpha_{m}(\pi')\beta_{m,G}(\pi'_{\cP_m},A) 
		\leq \sum_{\pi' \in G_m} \max_{A' \in \cA} \alpha_{m}(\pi') \beta_{m,G}(\pi'_{\cP_m},A').
		\end{align*}
		Thus it suffices to prove \eqref{Galphabeta} for every $m \in I_{n-1,A}$ and $\pi' \in \{0,1\}^{\overline{\cP_n}}$.
		
		Let $m \in I_{n-1,A}$ and $\pi' \in \{0,1\}^{\overline{\cP_n}}$.
		Let us define 
		\begin{align*}
		B_{m+1, n}(A) := \{ \pi'' \in \{0,1\}^{[m+1,n-1]} \mid \pi''_{k} = A_k \text{ if } k \in [m+1,n-1] \text{ and } A_k \neq *  \}.
		\end{align*}
		For $\pi' \in \{0,1\}^{\overline{\cP_m}}$,  
		observe that 
		\begin{align*}
		&\{\pi'' \in B_n(\pi, A) \mid \pi''_{\overline{ \cP_m} } = \pi' \}\\
		&= \left\{\pi'' \in \{0,1\}^{n-1} \left| 
		\begin{array}{l}
		\pi''_{\overline{ \cP_m} } = \pi',  \pi''_{\cP_n } = \pi ,\\ 
		\pi''_k=A_k \text{ if } A_k \neq * \text{ and } k \in [1,n-1]
		\end{array}
		\right\}\right. \\ 
		&\subseteq \left\{\pi'' \in \{0,1\}^{n-1} \left| 
		\begin{array}{l}
		\pi''_{\overline{ \cP_m} } = \pi'\\ 
		\pi''_k=A_k \text{ if } A_k \neq * \text{ and } k \in [1,n-1]
		\end{array}
		\right\}\right. \\  
		&= \{(\pi'', \pi'_m, \pi''')_{\rcat} \mid  \pi'' \in B_m(\pi'_{\cP_m},A), \pi''' \in B_{m+1, n}(A) \},
		\end{align*}
		where $(\pi'', \pi'_m, \pi''')_{\rcat} \in \{0,1\}^{n-1}$ is the concatenation
		of the three vectors with respective dimensions $m-1$, $1$, and $n-m-1$.
		Thus it holds that
		\begin{align}
		&\sum_{ \substack{\pi'' \in B_n(\pi, A)\\ : \pi''_i = \pi'_i , i \in \overline{ \cP_m} }} 
		\prod_{i \in I_{n-1,A}: i \leq m-1}\alpha_{i,G}(\pi''_{\overline{\cP_i}} )
		\prod_{j \in I_{n-1,A} : j \geq m+1} \alpha_{j}(\pi''_{\overline{\cP_i}} )\nonumber \\
		&\leq \sum_{\pi'' \in B_m(\pi'_{\cP_m},A) } \sum_{\pi''' \in B_{m+1, n}(A) } \prod_{i \in I_{n-1,A} : i \leq m-1} \alpha_{i,G} ( \pi''_{\overline{\cP_i}} ) 
		\prod_{j \in I_{n-1,A} : j \geq m+1} \alpha_{j}((\pi'', \pi'_m, \pi''')_{\rcat,\overline{\cP_j}})
		\nonumber \\
		&= \sum_{\pi'' \in B_m(\pi'_{\cP_m},A) } \prod_{i \in I_{m-1,A}}\alpha_{i,G}(\pi''_{\overline{\cP_i}}) 
		\left( 
		\sum_{\pi''' \in B_{m+1, n}(A) } \prod_{j \in I_{n-1,A} : j \geq m+1} \alpha_{j}((\pi'', \pi'_m, \pi''')_{\rcat,\overline{\cP_j}}) 
		\right) \nonumber \\
		&= \sum_{ \pi'' \in B_m(\pi'_{\cP_m}, A)  } \prod_{i \in I_{m-1,A}} \alpha_{i,G}(\pi''_{\overline{\cP_i}} ) \nonumber \\
		&= \beta_{m,G}(\pi'_{\cP_m},A). \label{lem8_bound}
		\end{align}
		The second equality holds since for each $\pi'' \in B_m(\pi'_{\cP_m}, A) $, 
		the sum in the above parenthesis is equal to $1$. The inequality \eqref{Galphabeta} and thus (i) hold.
		
		(ii) For any $m \in [1,N]$ and $\pi \in \{0,1\}^{\overline{\cP_m}}$, we can apply the discussion for proving \eqref{lem8_bound} to show
		\begin{align}
		\sum_{ \substack{\pi' \in B(A)\\ : \pi'_i = \pi_i , i \in \overline{ \cP_m} }} 
		\prod_{i \in I_{N,A}: i \neq m}\alpha_{i,G}(\pi'_{\overline{\cP_i}} )
		&\leq \sum_{ \substack{\pi' \in B(A)\\ : \pi'_i = \pi_i , i \in \overline{ \cP_m} }} 
		\prod_{i \in I_{N,A}: i \leq m-1}\alpha_{i,G}(\pi'_{\overline{\cP_i}} )
		\prod_{j \in I_{N,A} : j \geq m+1} \alpha_{j}(\pi'_{\overline{\cP_j}} ) \nonumber \\
		&\leq  \beta_{m,G}(\pi_{\cP_m},A) \label{betamG}
		\end{align}
		For $m=0,1,\dots,N$, let $D^m$ be defined by 
		\begin{align*}
		D^m := \left\{ D^m_{1}, D^m_{2},\dots,D^m_{N} \left| D^m_i = 
		\begin{cases}
		D_i \quad \text{if } i = m \\
		G_i \quad \text{otherwise.}
		\end{cases}  
		\right\} \right.
		\end{align*}
		We define $\mu_{G^m}$ and $\mu_{D^m}$ by replacing $D$ by $G^m$ and $D^m$ in the definition of $\mu_{D}$, respectively.
		Then it holds that
		\begin{align*}
		\mu(A) - \mu_{D} (A) &= \sum_{m=1}^{N} ( \mu_{G^{m-1}}(A) - \mu_{G^m}(A) ) + (\mu_{G}(A) - \mu_{D}(A)) \\
		&\leq \sum_{m=1}^{N} ( \mu_{G^{m-1}}(A) - \mu_{G^m}(A) ) + \sum_{m=1}^{N} ( \mu_{G}(A) - \mu_{D^m}(A) ).
		\end{align*}
		For the first term, observe that
		\begin{align*}
		\mu_{G^{m-1}}(A) - \mu_{G^m}(A) &= \sum_{\pi \in G_m} \alpha_{m}(\pi) \sum_{ \substack{\pi' \in B(A)\\ : \pi'_{\overline{\cP_m}} = \pi }} \prod_{i \in I_{N,A}: i \leq m-1}\alpha_{i,G}(\pi'_{\overline{\cP_i}} )
		\prod_{j \in I_{N,A} : j \geq m+1} \alpha_{j}(\pi'_{\overline{\cP_j}} )\\
		&\leq \sum_{\pi \in G_m} \alpha_{m}(\pi) \beta_{m,G}(\pi_{\cP_m},A) \\
		&\leq \sum_{\pi \in G_m} \max_{A' \in \cA} \alpha_{m}(\pi) \beta_{m,G}(\pi_{\cP_m},A').
		\end{align*}
		The first inequality follows from \eqref{betamG}.
		Thus we have 
		\begin{align*}
		\sum_{m=1}^{N} ( \mu_{G^{m-1}}(A) - \mu_{G^m}(A) ) \leq \sum_{m=1}^N \sum_{\pi \in G_m} \max_{A \in \cA} \alpha_{m}(\pi) \beta_{m,G}(\pi,A).
		\end{align*}
		For the second term, since $D_m \setminus G_m \subseteq  H_m$, we have
		\begin{align*}
		\mu_{G}(A) - \mu_{D^m}(A)
		&= \sum_{\pi \in H_m} \alpha_{m}(\pi) \sum_{ \substack{\pi' \in B(A)\\ : \pi'_{\overline{\cP_m}} = \pi }} \prod_{i \in I_{N,A} : i \neq m}\alpha_{i,G}(\pi'_{\overline{\cP_i}} )\\
		&\leq \sum_{\pi \in H_m} \alpha_m(\pi) \beta_{m,G}(\pi_{\cP_m},A)\\
		&\leq \sum_{\pi' \in H^{\downarrow}_m} \beta_{m,G}(\pi',A)\\
		&\leq  \sum_{\pi' \in H^{\downarrow}_m} \max_{A \in \cA} \beta_{m,G}(\pi',A).
		\end{align*}
		The first inequality follows from \eqref{betamG}.
		Thus we have (ii).
	\end{proof}
\end{lem}
The third lemma bounds the relative error of $\hat{\beta}$.
This statement can be proven by induction on the basis of Lemma~\ref{lem_bound1}.
\begin{lem}\label{lem_induction}
	The following holds for every $n=1,2,\dots,N$, $\overline{\pi} \in \{0,1\}^{\overline{\cP_n}}$ with $\pi = \overline{\pi}_{\cP_n}$, and $A \in \cA$ with a probability of at least $1-6C/T$:
	\begin{align}
	&\frac{1}{e} \beta_{n,G}(\pi, A) \leq \hat{\beta}_{n}(\pi, A) \leq  e \beta_{n,G}(\pi, A), \label{beta_ratio} \\
	&\overline{\pi}  \in G_n \quad \text{if } \alpha_n(\overline{\pi}) \beta_{n}(\pi,\hat{A}_{n,\pi}) < S(\lambda),\label{Gcondition} \\
	&\left( 1 - \frac{1}{\sqrt{\lambda} N} \right) \alpha_{n} (\overline{\pi})
	\leq \check{\alpha}'_{n} (\overline{\pi}) 
	\leq \left( 1 + \frac{1}{ \sqrt{\lambda} N} \right) \alpha_{n} (\overline{\pi}) \qquad \text{if } \alpha_n(\overline{\pi}) \beta_{n}(\pi,\hat{A}_{n,\pi}) \geq S(\lambda). \label{alpha_sandwich2}
	\end{align}
	\begin{proof}
		We suppose that  \eqref{beta_ratio}, \eqref{Gcondition}, and \eqref{alpha_sandwich2} hold for every $n = 1,2,\dots,N'-1$, $\overline{\pi} \in \{0,1\}^{\overline{ \cP_n}}$, and $A \in \cA$,
		and will prove that,
		for $n=N'$ and a single $\overline{\pi} \in \{0,1\}^{\overline{\cP}_{N'}}$
		with $\pi =\overline{\pi}_{\cP_{N'}}$,
		\begin{itemize}
			\item \eqref{beta_ratio} holds for all $A \in \cA$,
			\item \eqref{Gcondition} and \eqref{alpha_sandwich2} hold with probability $1 - 3/ T$.
		\end{itemize}
		Since $\sum_{n=1}^N |\overline{ \cP_n}| = 2C$, 
		this indicates that \eqref{beta_ratio}, \eqref{Gcondition}, and \eqref{alpha_sandwich2} 
		hold for all combinations of $n$, $\pi$, and $A$
		with probability at least $1 - 6C/ T$.

		For \eqref{beta_ratio}, 
		observe that $\pi' \in G_n$ implies $\alpha_{n,G}(\pi') = \check{\alpha}_n(\pi') = 0$.
		Thus by \eqref{Gcondition} and \eqref{alpha_sandwich2},
		the following holds for every $n=1,2,\dots,N'-1$ and $\pi' \in \{0,1\}^{\overline{ \cP_{n}}}$, irrespective of 
		whether $\alpha_n(\pi') \beta_{n}(\pi'_{\cP_n},\hat{A}_{n,\pi}) < S(\lambda)$ or not:
		\begin{align}
		\check{\alpha}_{n} (\pi') 
		\leq \left( 1 + \frac{1}{ \sqrt{\lambda} N} \right) \alpha_{n,G} (\pi') . \label{alpha_sandwich}
		\end{align}
		Then for every $A \in \cA$, we have
		\begin{align*}
		&\hat{\beta}_{N'}(\pi, A) = \sum_{\pi' \in B_{N'}(\pi,A) }\prod_{m \in I_{N'-1,A} } \check{\alpha}_{m}(\pi'_{\overline{\cP_m}})\\
		&\leq \left( 1 + \frac{1}{\sqrt{\lambda}N} \right)^{N'-1} \sum_{\pi' \in B_{N'}(\pi,A) }
		\prod_{m \in I_{N'-1,A}} \alpha_{m,G}(\pi'_{\overline{\cP_m}})\\
		&\leq e \beta_{N',G}(\pi, A).
		\end{align*}
		The first inequality holds from \eqref{alpha_sandwich},
		and the second inequality holds since $\lambda \geq 1$ and $(1+1/N)^{N'-1} \leq e$ if $N' \leq N$. 
		Thus the right inequality of \eqref{beta_ratio} holds.
		Since $(1 - 1/N)^{N'-1} \geq 1/e$,
		the left inequality is shown in the same way. Thus \eqref{beta_ratio} holds for $n'=N$ and all $A \in \cA$.
		
		For \eqref{Gcondition}, suppose that $\alpha_{N'}(\overline{\pi}) \beta_{N'}(\pi,\hat{A}_{N',\pi}) < S(\lambda)$.
		Then, since 
		\begin{align*}
		\hat{\beta}_{N'} (\pi,\hat{A}_{N',\pi})\leq e \beta_{N',G} (\pi,\hat{A}_{N',\pi}) \leq e \beta_{N'} (\pi,\hat{A}_{N',\pi})
		\end{align*}
		by \eqref{beta_ratio}, 
		Lemma~\ref{lem_bound1} (i) implies that $\check{\alpha}'_{N'} (\overline{\pi}) \hat{\beta}_{N'} (\pi,\hat{A}_{N',\pi}) 	\leq  2e S(\lambda)$ with probability at least $1-2/T$.
		This implies $\overline{\pi} \in G_n \subseteq D_n$ by Line~15 of Algorithm~\ref{algo_est1}, and thus \eqref{Gcondition} holds.

		For \eqref{alpha_sandwich2}, if $\alpha_{N'}(\overline{\pi}) \beta_{N'}(\pi, \hat{A}_{N',\pi}) \geq  S(\lambda)$, then by Lemma~\ref{lem_bound1} (ii) we have \eqref{alpha_sandwich2} with probability at least $1 - 3/T$.
		
		As a result, we have  \eqref{Gcondition} and \eqref{alpha_sandwich2}
		for $n=N'$ and a single $\overline{\pi}$
		with probability at least $1 - 3/T$.
		Hence 
		\eqref{beta_ratio}, \eqref{Gcondition}, and \eqref{alpha_sandwich2} 
		hold for all $n$, $\overline{\pi}$, and $A \in \cA$
		with probability $1-6C/T$.
	\end{proof}
\end{lem}
We are now ready to prove Proposition~\ref{prop_est} on the basis of Lemmas~\ref{lem_bound1}--\ref{lem_induction}.
\begin{proof}[Proof of Proposition~\ref{prop_est}]
	By Lemma~\ref{lem_induction}, \eqref{beta_ratio}, \eqref{Gcondition}, and \eqref{alpha_sandwich2} hold with a probability of $1 - 6C/T$.
	
	\eqref{beta_ratio1} directly follows from \eqref{beta_ratio}, since $\beta_{n,D}(\pi, A)/e \leq \beta_{n,G}(\pi, A)/e $ and $e \beta_{n,G}(\pi, A) \leq   e \beta_{n}(\pi, A)$.
	
	\eqref{alphabeta1} also directly follows from the contraposition of \eqref{Gcondition}, which states that: if $\overline{\pi} \not \in G_n$, then $\alpha_n(\overline{\pi}) \beta_{n}(\pi,\hat{A}_{n,\pi}) \geq S(\lambda)$.
	Recall that we are claiming  \eqref{alphabeta1} only for $\overline{\pi} \in \{0,1\}^{\overline{\cP_n}} \setminus D_n$. Since 
	$G_n \subseteq D_n$,
	$\overline{\pi} \not\in D_n$ indicates $\overline{\pi} \not\in G_n$.
	Hence the contraposition of \eqref{Gcondition} implies 
	\eqref{alphabeta1} for  $\overline{\pi} \in \{0,1\}^{\overline{\cP_n}} \setminus D_n$.
	
	To prove \eqref{beta_diff}, 
	first we show that for every $n \in [1,N]$, $\overline{\pi} \in \{0,1\}^{\overline{\cP_n}}$, and $\pi = \overline{\pi}_{\cP_n}$ it holds that 
	\begin{align}
	\alpha_{n}(\overline{\pi}) \hat{\beta}_{n} (\pi, \hat{A}_{n,\pi}) \leq 4e S(\lambda) \quad \text{if }\overline{\pi} \in G_n.\label{cond_alphabeta_pre} 
	\end{align}
	We prove the contraposition. Suppose that $\alpha_{n}(\overline{\pi}) \hat{\beta}_{n} (\pi, \hat{A}_{n,\pi}) > 4e S(\lambda)$. 
	Since $\check{\alpha}'_{n}(\overline{\pi}) \geq (1-1/N)\alpha_{n}(\overline{\pi}) $ by \eqref{alpha_sandwich2},
	%	Then $\alpha_{n}(\overline{\pi}) \beta_{n} (\pi, \hat{A}_{n,\pi}) > 4 e S(\lambda)$ by \eqref{beta_ratio}.
	%	Then $\alpha_{n}(\overline{\pi}) \beta_{n} (\pi, \hat{A}_{n,\pi}) > 4 e S(\lambda)$ by \eqref{beta_ratio}, and thus $\check{\alpha}'_{n}(\overline{\pi}) \geq (1-1/N)\alpha_{n}(\overline{\pi}) $ by \eqref{alpha_sandwich2}.
	we have
	\begin{align*}
	\check{\alpha}'_{n}(\overline{\pi}) \hat{\beta}_{n}(\pi, \hat{A}_{n,\pi}) &\geq (1-1/N)  \alpha_{n}(\overline{\pi})\hat{\beta}_{n}(\pi, \hat{A}_{n,\pi})  \\
	&\geq (1-1/N) 4e S(\lambda) >  2 e S(\lambda),
	%	\check{\alpha}'_{n}(\overline{\pi}) \hat{\beta}_{n}(\pi, \hat{A}_{N',\pi}) &\geq e^{-1}\check{\alpha}'_{n}(\overline{\pi}) \hat{\beta}_{n, G}(\pi, \hat{A}_{N',\pi}) \\
	%	&\geq (1-1/N)e^{-2}\alpha_{n}(\overline{\pi}) \beta_{n, G}(\pi, \hat{A}_{N',\pi}) \\
	%	&>  2 S(\lambda),
	\end{align*}
	and thus $\overline{\pi} \not \in G_n$.
	
	By Lemma~\ref{lem_delete} (i), it holds that
	\begin{align*}
	\beta_{n}(\pi,A) \leq 
	\beta_{n,G}(\pi,A) 
	+ \sum_{m=1}^{N} \sum_{\pi' \in G_m} \max_{A' \in \cA} \alpha_m(\pi')\beta_{m,G}(\pi'_{\cP_m}, A').
	\end{align*}
	Thus it suffices for proving \eqref{beta_diff} to show
	that, for $\pi \not \in D_n$, $\beta_{n,G}(\pi,A) \leq e \hat{\beta}_n(\pi,A)$ and 
	\begin{align}
	\label{eq.beta}
	\sum_{m=1}^N \sum_{\overline{\pi}' \in G_m} \max_{A \in \cA} \alpha_{m}(\overline{\pi}') \beta_{m, G}(\pi', A)  
	\leq \frac{e \hat{\beta}_n(\pi,\hat{A}_{n,\pi})}{C},
	\end{align}
	where $\pi'$ denotes $\overline{\pi}'_{\cP_n}$.
	Since the former follows from \eqref{beta_ratio},
	\eqref{eq.beta} remains to be proven.
	
	Observe that $\max_{A \in \cA} \beta_{n,G}(\pi',A) \leq e\hat{\beta}_{n} (\pi', \hat{A}_{n,\pi'})$ by \eqref{beta_ratio} and the definition of $\hat{A}_{n,\pi'}$.
	Then the following holds for every $n=1,2,\dots,N$ by \eqref{cond_alphabeta_pre}:
	\begin{align}
	\max_{A \in \cA} \alpha_{n}(\overline{\pi}') \beta_{n, G}(\pi', A) \leq 4e^2 S(\lambda)\quad \text{if }\overline{\pi}' \in G_n.  \label{cond_alphabeta} 
	\end{align}
	Since $\sum_{n=1}^N |G_n| \leq 2C$, it holds that
	\begin{align}
	\sum_{n=1}^N \sum_{\overline{\pi}' \in G_n} \max_{A \in \cA} \alpha_{n}(\overline{\pi}') \beta_{n, G}(\pi', A)  \leq 8e^2 C S(\lambda). \label{sumMaxAlphaBeta}
	\end{align}
	Since $\overline{\pi} \not \in D_n$, we have $\overline{\pi} \not\in H_n$. Hence, Line~20 of Algorithm~\ref{algo_est1} implies that $\hat{\beta}_{n}(\pi,\hat{A}_{n,\pi}) \geq 8e C^2 S(\lambda)$, which is equivalent to
	\begin{align}
	8e^2 C S(\lambda) \leq \frac{e\hat{\beta}_{n}(\pi,\hat{A}_{n,\pi})}{C} . \label{hatBetaBound}
	\end{align}
	Therefore, \eqref{sumMaxAlphaBeta} and \eqref{hatBetaBound} imply \eqref{eq.beta}.
	
	To prove \eqref{muD_diff}, 
	we bound the right-hand side of the inequality given in 
	Lemma~\ref{lem_delete} (ii).
	Its first term is bounded by \eqref{sumMaxAlphaBeta}.
	To bound the second term,
	observe that 
	the following holds for every $n=1,2,\dots,N$ and $\pi' \in H_n^{\downarrow}$,
	by \eqref{beta_ratio} and Line~20 of Algorithm~\ref{algo_est1}: 
	\begin{align}
	\max_{A \in \cA} \beta_{n,G}(\pi',A) 
	&\leq  \max_{A \in \cA} e \hat{\beta}_{n}(\pi',A) \nonumber \\
	&= e \hat{\beta}_{n}(\pi',\hat{A}_{n,\pi}) \leq 8e^2 C^2 S(\lambda). \label{cond_beta}
	\end{align}
	Since $\sum_{n=1}^N |H_n^{\downarrow}| \leq C$,  \eqref{cond_beta} implies 
	\begin{align}
	\sum_{n=1}^N \sum_{\pi' \in H_n^{\downarrow}}\max_{A \in \cA} \beta_{n,G}(\pi',A) \leq 8e^2 C^3 S(\lambda). \label{sumMaxBeta}
	\end{align}
	\eqref{muD_diff}
	is implied by Lemma~\ref{lem_delete} (ii), \eqref{sumMaxAlphaBeta}, and \eqref{sumMaxBeta}.
\end{proof}

\subsection{Accuracy of Algorithm~\ref{algo_est2}}\label{subsec_algo2}
This subsection bounds the gap between the true value $\mu(A)$ and its estimate $\hat{\mu}(A)$ given by Algorithm~\ref{algo_est2}, assuming that the input of Algorithm~\ref{algo_est2}, which is output of Algorithm~\ref{algo_est1}, satisfies the conditions in Proposition~\ref{prop_est}.
\begin{prop}\label{prop_bound}
	Suppose that $\lambda \geq 1$, and $\hat{\beta}_{n}$ and $D$ satisfy \eqref{beta_ratio1}, \eqref{alphabeta1}, \eqref{beta_diff}, and \eqref{muD_diff}.
	Let $\hat{\alpha}_{n}$ be the output of Algorithm~\ref{algo_est2},
	and let $\hat{\mu}$ be defined by \eqref{hat_mu_alpha}.
	Then the following holds for every $A \in \cA$ with a probability of at least $1 - (10C+2)/T$:
	\begin{align}	
	\left|\mu(A)  - \hat{\mu}(A) \right| 
	\leq \sqrt{\frac{2e^6 \gamma^* \log (|\cA|T) }{ T}} + \sqrt{ \frac{8e^2 C^3 \log  T}{\lambda  T}} 
	\quad + 8e^2(C^3 + C)S(\lambda) \label{eq_prop_bound}
	\end{align}
\end{prop}
Recall that $I_{N,A} := \{m \in [1,N] \mid A_m = * \}$ for $A\in \cA$.
For $n \in [1,N]$ and $\pi \in \{0,1\}^{\overline{\cP}_n}$, let $\Delta \alpha_n(\pi ) := \hat{\alpha}_n(\pi) - \alpha_{n,D}(\pi)$.
For $A \in \cA$ and $J \subseteq I_{N,A}$, we define $f^J(A)$ by 
\begin{align*}
f^{J}(A) 
= \sum_{\pi \in B(A)} \prod_{m \in I_{N,A} \setminus J}\alpha_{m,D}(\pi_{\overline{\cP_m}}) \prod_{n \in J} \Delta \alpha_n(\pi_{\overline{\cP_n}}) .
\end{align*}
Observe that $f^{J}(A)$ is given by replacing $\alpha_{n,D}(\pi_{\overline{\cP_n}})$ by $\Delta \alpha_n(\pi_{\overline{\cP_n}})$ for $n \in J$ in the definition \eqref{muD} of $\mu_{D}$.
Recall that $\hat{\mu}(A)$ is given by replacing $\alpha_{n,D}(\pi)$
in the definition of $\mu_D$
by $\hat{\alpha}_n(\pi)$ for all $n \in I_{N,A}$.
Based on these relationships, we have the following lemma:
\begin{lem}\label{lem_f}
	For $A \in \cA$, it holds that: 
	\begin{align}
	\mu_{D}(A) &= f^{\emptyset}(A), \nonumber \\
	\hat{\mu}(A) &= \sum_{J \subseteq I_{N,A}} f^{J} (A). \label{hatmuAf}
	\end{align}
	\begin{proof}
		The first equality directly follows from the definition of $f^\emptyset$.
		For the second, 
		\begin{align*}
		\hat{\mu}(A) 
		&= \sum_{\pi \in B(A)} \prod_{n \in I_{N,A}}(\alpha_{n,D}(\pi_{\overline{\cP_n}}) + \Delta \alpha_{n}(\pi_{\overline{\cP_n}}))\\
		&= \sum_{\pi \in B(A)} \sum_{J \subseteq I_{N,A}} \prod_{n \in I_{N,A} \setminus J  }\alpha_{n,D}(\pi_{\overline{\cP_n}}) \prod_{n \in J}\Delta \alpha_{n}(\pi_{\overline{\cP_n}})\\
		&=  \sum_{J \subseteq I_{N,A}} \sum_{\pi \in B(A)} \prod_{n \in I_{N,A} \setminus J  }\alpha_{n,D}(\pi_{\overline{\cP_n}}) \prod_{n \in J}\Delta \alpha_{n}(\pi_{\overline{\cP_n}})\\
		&= \sum_{J \subseteq I_{N,A}} f^{J} (A).
		\end{align*}
		The second equality holds by the binary expansion of $\prod_{n \in I_{N,A} }(\alpha_{n,D}(\pi) + \Delta \alpha_{n}(\pi))$.
	\end{proof}
\end{lem}
For $j \in I_{N,A}$, let $f^{j}(A) := f^{\{j\}}(A)$.
We provide probabilistic bounds for the linear terms ($|J|=1$) and super-linear terms ($|J|\geq 2$) in \eqref{hatmuAf}, separately.
\begin{lem}\label{lem_bound2}
	Suppose that \eqref{beta_ratio1}, \eqref{alphabeta1}, and \eqref{beta_diff} hold.
	
	(i) The following holds with a probability of at least $1 - (C+2)/T$:
	\begin{align*}
	\max_{A \in \cA} \left| \sum_{j \in I_{N,A}} f^j(A) \right| 
	\leq \sqrt{  \frac{2 e^6 \gamma^* \log (|\cA|T)}{ T}}.
	\end{align*}
	
	(ii) The following holds with a probability of at least $1 - 9C/T$:
	\begin{align*}
	\max_{A \in \cA} \sum_{J \subseteq  I_{N,A}: |J| \geq 2} |f^J(A)|
	\leq \sqrt{ \frac{ 8e^2 C^3 \log T}{\lambda T}}.
	\end{align*}
	\begin{proof}
		For $t=1,2,\dots,2T/3$,
		let $A_t$ be the intervention applied in the $t$-th experiment in Algorithm~\ref{algo_est2},
		and $\omega^{(t)}$ be the corresponding realizations.
		For each $n \in [1,N]$, $\pi \in \{0,1\}^{\cP_n}$,
		and $t=1,2,\dots,2T/3$,
		we define a random variable $X_{n,\pi,t}$ by
		\begin{align*}
		X_{n,\pi,t} = 
		\begin{cases}
		1 \quad \text{ if } \pi_i = \omega^{(t)}_{i} \text{ for } \forall i \in \cP_n \text{ and } A_{t,n}=* \\
		0 \quad \text{ otherwise.}
		\end{cases}
		\end{align*}
		We also define a random variable $T_n(\pi) := \sum_{t=1}^{2T/3} X_{n,\pi,t}$.
		
		Let $n \in [1,N]$ and $\pi\in \{0,1\}^{\cP_n}$.
		If the $t$-th experiment adapts $\hat{A}_{n,\pi}$, then
		$E[X_{n,\pi,t}] = \beta_n(\pi,\hat{A}_{n,\pi})$
		holds.
		Among the first $T/3$ experiments done in the algorithm,
		at least $T/(3C)$ experiments adapt $\hat{A}_{n,\pi}$.
		Moreover,
		if $t \geq T/3+1$,
		then the $t$-th experiment adapts $A'$ sampled from $\cA$
		according to the probability $\hat{\mu}$.
		% Recall that $\hat{\eta}$ is an optimal solution for \eqref{optProb}.
		%		For $n \in [1,N]$ and $\pi \in \{0,1\}^{\cP_n}$, let us define 
		Hence 
		$E[X_{n,\pi,t}]\geq \sum_{A' \in \cA}\hat{\eta}_{A'}\beta_n(\pi,A')$
		holds for
		each $t \geq T/3+1$.
		These imply that
		\begin{align}
		E[T_n(\pi)] \geq T_{n,\pi}, \label{expT}
		\end{align}
		holds,
		where 
		\begin{align}
		T_{n,\pi} := \frac{\beta_n(\pi,\hat{A}_{n,\pi})T}{3C} + \frac{T}{3}\sum_{A' \in \cA}\hat{\eta}_{A'}\beta_n(\pi,A') . \label{Tnpi}
		\end{align}		
		
		From $X_{n',\pi',t}$ ($1 \leq n' \leq n$, $\pi' \in
		\{0,1\}^{\cP_{n'}}$, $t=1,2,\dots, 2T/3$),
		we define a random variable $Y_{n,\pi,k}$ as follows for each $k=1,2,\dots,T_n(\pi)$:
		let $1 \leq t_1 < t_2 < \cdots t_{T_n(\pi)} \leq 2T/3$ be the indices such that $X_{n, \pi,t_k} = 1$ for $k=1,2,\dots, T_n(\pi)$;
		then 
		\begin{align*}
		Y_{n,\pi,k} = 
		\begin{cases}
		1 \quad \text{ if } \omega^{(t_k)}_{n} = 1 \\
		0 \quad \text{ otherwise.}
		\end{cases}
		\end{align*}
		Let $\overline{T_n}(\pi) := \sum_{k=1}^{T_n(\pi)}Y_{n,\pi,k}$.
		
		Let $\overline{\pi}^{0} \in \{0,1\}^{\overline{\cP_n}}$ (resp.,
		$\overline{\pi}^{1} \in \{0,1\}^{\overline{\cP_n}}$)
		be the extension of $\pi \in \{0,1\}^{\cP_n}$ such that 
		$\overline{\pi}^{0}_n = 0$
		(resp., $\overline{\pi}^{1}_n = 1$).
		Then it holds that
		\begin{align}\label{eq:alg2_hat_alpha}
		E[\overline{T_n}(\pi) \mid T_{n}(\pi)] \geq
		\alpha_n(\overline{\pi}^{1}) T_n(\pi)\quad\text{ and }\quad
		\hat{\alpha}_n(\overline{\pi}^{1}) = \frac{\overline{T_n}(\pi) }{T_{n}(\pi)}.
		\end{align}
		We also introduce $\tau_n(\pi)$ by
		\begin{align}
		\tau_n(\pi) := \frac{3E[T_n(\pi)]}{4} \geq \frac{3}{4} T_{n,\pi}, \label{lem_bound21} 
		\end{align}
		where the inequality follows from \eqref{expT}.
		If $\pi \not \in D^{\downarrow}_n$, then $\overline{\pi}^{0} \not \in D_n$
		or  $\overline{\pi}^{1} \not \in D_n$ holds.
		Then by \eqref{alphabeta1}, \eqref{expT}, and \eqref{Tnpi}, it holds that
		\begin{align*}
		E[T_n(\pi)]\geq \frac{ \beta_n(\pi,\hat{A}_{n,\pi})T}{3C}   \geq \frac{\alpha_n(\overline{\pi}^{k}) \beta_n(\pi,\hat{A}_{n,\pi})T}{3C} \geq \frac{TS(\lambda
			)}{3C} = 4\lambda N^2 \log T \geq  32\log T,
		\end{align*}
		since $\lambda \geq 1$ and $N \geq 3$ by the assumption.
		Applying Proposition~\ref{prop_chernoff} to $T_n(\pi) = \sum_{t=1}^{2T/3} X_{n,\pi,t}$ with $\delta = 1/4$, 
		we have for $\pi \not \in D^{\downarrow}_n$, 
		\begin{align}
		\Prob \left(T_n(\pi) \geq \tau_n(\pi) \right) \geq 1 - \frac{1}{T}. \label{Ttau}
		\end{align}
		for $n=1,2,\dots,N$ and $\pi \not \in D^{\downarrow}_n$.
		
		(i) First we prove that the following holds for every $A$ with probability $1 - (C+2)/T$:
		\begin{align}
		\left| \sum_{j \in I_{N,A}} f^j(A) \right| 
		\leq \sqrt{ \sum_{n \in I_{N,A}} \sum_{\pi \in \{0,1\}^{\cP_n} \setminus D_n^{\downarrow}} \frac{2 \beta_{n,D}^2(\pi,A)\log (|\cA|T)}{3 T_{n,\pi}}}. \label{prob(i)}
		\end{align}
		For $j \in I_{N,A}$, define
		\begin{align}
		\beta'_{j} (\pi,k,A) &:= \sum_{\substack{ \pi' \in B(A): \\ \pi'_{\cP_j} = \pi, \pi_j = k } } \prod_{m \in I_{N,A} : m \neq j}\alpha_{m,D}(\pi'_{\overline{\cP_m}}) \qquad (k=0,1) \label{dfBetaPra} \\
		\beta'_{j} (\pi,A) &:= \beta'_{m} (\pi,1,A) -  \beta'_{m} (\pi,0,A). \nonumber
		\end{align}
		Since $\Delta\alpha_{j}(\overline{\pi}^{1}) = - \Delta\alpha_{j}(\overline{\pi}^{0})$ for any $\pi \in \{0,1\}^{\cP_n}$, we have
		\begin{align*}
		f^j(A) 
		&= \sum_{\pi \in B(A)} \Delta\alpha_{j}(\pi_{\overline{\cP_j}}) \prod_{m \in I_{N,A} : m \neq j}\alpha_{m,D}(\pi_{\overline{\cP_m}}) \\
		&= \sum_{\pi \in \{0,1\}^{\cP_j}} \left( \Delta\alpha_{j}(\overline{\pi}^{1}) \beta'_{j} (\pi,1,A)  + \Delta\alpha_{j}(\overline{\pi}^{0}) \beta'_{j} (\pi,0,A)   \right) \\
		&= \sum_{\pi \in \{0,1\}^{\cP_j}} \left( \Delta\alpha_{j}(\overline{\pi}^{1}) \beta'_{j} (\pi,1,A)  - \Delta\alpha_{j}(\overline{\pi}^{1}) \beta'_{j} (\pi,0,A)   \right) \\
		&= \sum_{\pi \in \{0,1\}^{\cP_j}} \Delta\alpha_{j}(\overline{\pi}^{1}) \beta'_{j} (\pi,A)  \\
		&= \sum_{\pi \in \{0,1\}^{\cP_j} \setminus D^{\downarrow}_j} \Delta\alpha_{j}(\overline{\pi}^{1}) \beta'_{j} (\pi,A)\\
		&= \sum_{\pi \in \{0,1\}^{\cP_j} \setminus D^{\downarrow}_j}   \beta'_{j} (\pi,A) \left( \frac{1}{T_j(\pi)}\sum_{k=1}^{T_j(\pi)} Y_{j,\pi,k} - \alpha_{j,D}(\overline{\pi}^1) \right).
		\end{align*}
		The fifth equality holds since $\pi \in D^{\downarrow}_j$ implies
		$\alpha_{j,D}(\overline{\pi}^{1}) = \hat{\alpha}_{j,D}(\overline{\pi}^{1}) = \Delta\alpha_{j}(\overline{\pi}^{'1}) = 0$.
		Let us define $\varepsilon_A >0$ for $A \in \cA$ by
		\begin{align*}
		\varepsilon_A := \sqrt{ \sum_{j \in I_{N,A}} \sum_{\pi \in \{0,1\}^{\cP_j} \setminus D^{\downarrow}_j } \frac{\beta^{'2}_{j} (\pi,A)}{2 \tau_j(\pi)}\log (|\cA| T) }.
		\end{align*}
		Let us consider a unique correspondence between indices $m = 1,2,\dots, M$ in Lemma~\ref{lem_hoeffding}
		and pairs of elements $(j,\pi)$ where $j =1,2,\dots,N$ and $\pi \in \{0,1\}^{\cP_j} \setminus D^{\downarrow}_j$, satisfying the following condition:
		if $m \sim (j,\pi)$, $m' \sim (j',\pi')$, and $j < j'$, then it holds that $m < m'$.
		Under this relationship, the random variables $Y_{j,\pi,k}$ and $T_j(\pi)$, respectively corresponding to $Y_{m,k}$ and $T_m$ in Lemma~\ref{lem_hoeffding}, and the set $\cY_{j}(\pi) :=\{Y_{j,\pi,k} \mid k=1,2,\dots,T_j(\pi)\}$ of random variables satisfies the conditional independence assumption of Lemma~\ref{lem_hoeffding}.
		Thus we apply Lemma~\ref{lem_hoeffding} with constants $\varepsilon_A$ and $\tau_m = \tau_{j}(\pi)$ for $A \in \cA$, $j=1,2,\dots,N$, and $\pi \in \{0,1\}^{\cP_j} \setminus D^{\downarrow}_j$, it holds that
		\begin{align*}
		&\Prob \left(\left| \sum_{j \in I_{N,A}} f^j(A) \right| 
		\leq \varepsilon_A, \forall A \in \cA \right)\\
		&=\Prob \left(\left| \sum_{j \in I_{N,A}} \sum_{\pi \in \{0,1\}^{\cP_j} \setminus D^{\downarrow}_j}   \beta'_{j}(\pi,A) \left(  \frac{1}{T_j(\pi)}\sum_{k=1}^{T_j(\pi)} Y_{j,\pi,k} - \alpha_{j, D}(\overline{\pi}^1) \right) \right| 
		\leq \varepsilon_A, \forall A \in \cA \right)\\
		&\geq \Prob(T_j(\pi) \geq \tau_j(\pi), j \in I_{N,A}, \pi \in  \{0,1\}^{\overline{ \cP_n}} \setminus D^{\downarrow}_j)\notag \\
		& \quad - 2 \sum_{A \in \cA} \exp \left( - \frac{2\varepsilon_A^2}{ \sum_{j \in I_{N,A}} \sum_{ \pi \in  \{0,1\}^{\cP_j} \setminus D^{\downarrow}_j}  \beta^{'2}_{n} (\pi,A)/\tau_j(\pi) } \right)\\
		&\geq 1 - \frac{C}{T} - 2|\cA| \frac{1}{|\cA|T} \geq 1 - \frac{C + 2}{T}.
		\end{align*}
		The last inequality holds since \eqref{Ttau} holds with a probability of at least $1 - 1/T$ for each $j=1,2,\dots,N$ and $\pi \in \{0,1\}^{\cP_j} \setminus D^{\downarrow}_j$, where the number of such pair $(j,\pi)$ is at most $C$.
		If 
		\begin{align}
		|\beta^{'}_{j} (\pi,A)|  \leq \beta_{j,D}(\pi, A), \label{betaAbsBound}
		\end{align}
		then it holds that 
		\begin{align*}
		\varepsilon_A \leq \sqrt{ \sum_{j \in I_{N,A}} \sum_{\pi \in \{0,1\}^{\cP_j} \setminus D^{\downarrow}_j } \frac{2 \beta_{j,D}^2 (\pi,A)}{3T_{j,\pi}}\log (|\cA| T) },
		\end{align*}
		which and \eqref{lem_bound21}
		show the probabilistic bound \eqref{prob(i)}.
		Thus it is suffices to prove \eqref{betaAbsBound}.
		
		Recall that, for $j=1,2,\dots,N$, $\pi \in \{0,1\}^{\cP_j}$, and $A \in \cA$, $B_j(\pi,A)$ defined in \eqref{df_Bn}
		is the set of realizations which coincides with $\pi$ and $A$.
		Since 
		\begin{align*}
		\{\pi \in B(A) \mid \pi_{\cP_j} = \pi', \pi_j = k  \}
		&= \left\{\pi \in \{0,1\}^N  \left| 
		\begin{array}{l}
		\pi_{\cP_j} = \pi', \pi_j = k, \pi_N = 1, \\
		\pi_l = A_l \text{ if } A_l \neq * \text{ and } l \in [1,N]
		\end{array}
		\right\} \right. \\
		&\subseteq \left\{\pi \in \{0,1\}^N  \left| 
		\begin{array}{l}
		\pi_{\cP_j} = \pi', \pi_j = k,  \\
		\pi_l = A_l \text{ if } A_l \neq * \text{ and } l \in [1,N]
		\end{array}
		\right\} \right. \\
		&= \{(\pi'' , k, \pi''' )_{\rcat} \mid \pi'' \in B_{j}(\pi',A), \pi''' \in B_{j+1, N+1}(A) \},
		\end{align*}
		it holds for $k=0,1$ that
		\begin{align}
		\beta'_{j} (\pi',k,A) &\leq  \sum_{\pi'' \in B_{j}(\pi',A)} \sum_{ \pi''' \in B_{j+1, N+1}(A)} \prod_{m \in I_{N,A} : m \neq j}\alpha_{m,D}((\pi'' , k, \pi''' )_{\rcat, \overline{\cP_m}})\nonumber \\
		&= \sum_{\pi'' \in B_{j}(\pi',A)} \prod_{m \in I_{j-1,A}} \alpha_{m,D}(\pi''_{\overline{\cP_m}}) \left( \sum_{ \pi''' \in B_{j+1, N+1}(A)} \prod_{l \in I_{N,A} \cap [j+1,N]} \alpha_{l,D}((\pi'' , k, \pi''' )_{\rcat, \overline{\cP_m}}) \right) \nonumber \\
		&= \sum_{\pi'' \in B_{j}(\pi',A)} \prod_{m \in I_{j-1,A}} \alpha_{m,D}(\pi''_{\overline{\cP_m}}) \nonumber \\
		&= \beta_{j,D}(\pi',A). \label{betaAbsBoundSub}
		\end{align}
		This implies \eqref{betaAbsBound}.

		We now prove (i) on the basis of \eqref{prob(i)}. Let $\eta^* \in [0,1]^{\cA}$ be the optimum solution corresponding to $\gamma^*$.
		Then the RHS of \eqref{prob(i)} can be bounded as:
		\begin{align}
		&\max_{A \in \cA} \sqrt{ \sum_{j \in I_{N,A}} \sum_{\pi \in \{0,1\}^{\cP_j} \setminus D^{\downarrow}_j } \frac{2 \beta^2_{j,D} (\pi,A)}{3T_{j,\pi}}\log (|\cA| T) } \nonumber \\
		&= \max_{A \in \cA}\sqrt{ \sum_{j \in I_{N,A}} \sum_{\pi \in \{0,1\}^{\cP_j}  \setminus D_j^{\downarrow}} \frac{2\beta_{j,D}^2(\pi,A)\log (|\cA|T)}{\beta_j(\pi, \hat{A}_{j,\pi})T / C + \sum_{A' \in \cA} \hat{\eta}_{A'} \beta_{j}(\pi,A') T }} \nonumber \\
		&\leq \max_{A \in \cA}\sqrt{ \sum_{j \in I_{N,A}} \sum_{\pi \in \{0,1\}^{\cP_j}  \setminus D_j^{\downarrow}} \frac{2 e^3 \hat{\beta}_{j}^2(\pi,A)\log (|\cA|T)}{\hat{\beta}_{j}(\pi, \hat{A}_{j,\pi})T / C + \sum_{A' \in \cA} \hat{\eta}_{A'} \hat{\beta}_{j}(\pi,A')T  }} \label{boundHatEta}
		\end{align}
		The first equality holds by the definition \eqref{Tnpi} of $T_{j,\pi}$. 
		The last inequality follows from the fact $\sum_{A\in \cA} \eta_A = 1$ and \eqref{beta_ratio1},
		which implies $\beta_{j,D}(\pi,A) \leq e\hat{\beta}_{j}(\pi,A)$, $\beta_j(\pi, \hat{A}_{j,\pi}) \geq \hat{\beta}_{j}(\pi, \hat{A}_{j,\pi})/e $, and $\hat{\eta}_{A'} \beta_{j}(\pi,A') \geq  \hat{\beta}_{j}(\pi,A') / e$.
		Observe then that, removing the square root and ignoring the coefficient $2e^3/T$,
		\eqref{boundHatEta} coincides with the subject of minimization in \eqref{optProb},
		whose optimum solution is $\hat{\eta}$.
		Thus the replacement of $\hat{\eta}$ by $\eta^*$ provides an upper-bound,
		and we have
		\begin{align}
		\eqref{boundHatEta}&\leq \max_{A \in \cA}\sqrt{ \sum_{j \in I_{N,A}} \sum_{\pi \in \{0,1\}^{\cP_j}  \setminus D_j^{\downarrow}} \frac{2 e^3 \hat{\beta}_{j}^2(\pi,A)\log (|\cA|T)}{\hat{\beta}_{j}(\pi, \hat{A}_{j,\pi})T / C + \sum_{A' \in \cA} \eta^*_{A'} \hat{\beta}_{j}(\pi,A')T  }} \nonumber \\
		&\leq \max_{A \in \cA}\sqrt{ \sum_{j \in I_{N,A}} \sum_{\pi \in \{0,1\}^{\cP_j}  \setminus D_j^{\downarrow}} \frac{2 e^4 \hat{\beta}_{j}^2(\pi,A)\log (|\cA|T)}{\sum_{A' \in \cA} \eta^*_{A'} \beta_{j}(\pi,A')T  }} \nonumber \\
		&\leq \max_{A \in \cA}\sqrt{ \sum_{j \in I_{N,A}} \sum_{\pi \in \{0,1\}^{\cP_j}  \setminus D_j^{\downarrow}} \frac{2 e^6 \beta_{j}^2(\pi,A)\log (|\cA|T)}{\sum_{A' \in \cA} \eta^*_{A'} \beta_{j}(\pi,A')T  }} \nonumber \\
		&\leq \max_{A \in \cA}\sqrt{ \sum_{j \in I_{N,A}} \sum_{\substack{ \pi \in \{0,1\}^{\cP_j} \\ : \beta_{j}(\pi,A) > 0 }} \frac{2 e^6 \beta_{j}^2(\pi,A)\log (|\cA|T)}{\sum_{A' \in \cA} \eta^*_{A'} \beta_{j}(\pi,A')T  }} \label{sqrtOpt} \\
		&= \sqrt{  \frac{2 e^6 \gamma^* \log (|\cA|T)}{ T }}. \nonumber
		\end{align}
		The second inequality follows from the application of \eqref{beta_diff} to the denominator,
		and the third inequality follows from \eqref{beta_ratio1} again,
		as $\hat{\beta}_{j}(\pi,A) \leq e\beta_{j}(\pi,A)$.
		Since $\{0,1\}^{\cP_j}  \setminus D_j^{\downarrow} \subseteq \{\pi \in  \{0,1\}^{\cP_j} \mid \beta_{j}(\pi,A) > 0  \}$ holds for every $j=1,2,\dots,N$,
		we have the last inequality.
		Then \eqref{sqrtOpt} is the subject of minimization in \eqref{optProb*},
		whose optimum solution is $\eta^*$ with optimum value $\gamma^*$.
		Thus we have the last equality,
		and the proof of (i) is complete.

		(ii) We first show that
		\begin{align}
		|\Delta \alpha_n (\pi)| \leq \frac{1}{\sqrt{\lambda}N}  \alpha_{n,D} (\pi) \label{boundDeltaAlphaAbs}
		\end{align}
		holds for every $n \in I_{N,A}$ and $\pi \in \{0,1\}^{\cP_n}$ with a probability of at least $1 - 6C/T$.
		Since the number of such pair $(n,\pi)$ is $2C$,
		it suffices to show that for each $(n,\pi)$, \eqref{boundDeltaAlphaAbs} holds with a probability of at least $1 - 3/T$.
		If $\pi \in D_n$, which implies $\alpha_{n,D} (\pi) = \Delta \alpha_n (\pi) = 0$,
		then the above is trivial.
		If $\pi \not \in D_n$, then \eqref{alphabeta1} holds. 
		Similarly to the discussion as Lemma~\ref{lem_bound1} (ii), the following holds for each $(n,\pi)$ with probability $1 - 3/T$:
		\begin{align*}
		\left( 1 - \frac{1}{\sqrt{\lambda}N} \right) \alpha_{n} (\pi) 
		\leq \hat{\alpha}_{n} (\pi) 
		\leq  \left( 1 + \frac{1}{\sqrt{\lambda}N} \right) \alpha_{n} (\pi) .
		\end{align*}
		In fact, by \eqref{num_sample}, \eqref{expT}, and \eqref{Tnpi}, the expected number of samples available for estimating $\hat{\alpha}_{n} (\pi)$ in Algorithm~\ref{algo_est2} is larger than that for $\check{\alpha}'_{n} (\pi)$ in Algorithm~\ref{algo_est1}, and thus the same probabilistic bound on estimation error holds.
		Thus, whichever $\pi \in D_n$ or not, \eqref{boundDeltaAlphaAbs}
		holds with a probability of at least $1 - 3/T$ for a pair of
		$(n,\pi)$.
		
		Suppose that the inequality \eqref{boundDeltaAlphaAbs} holds for every $n \in I_{N,A}$ and $\pi \in \{0,1\}^{\cP_n}$,
		which is attained with a probability of at least $1 - 6C/T$.
		Observe that
		\begin{align}
		&\sum_{J \subseteq  I_{N,A}: |J| \geq 2} |f^J(A)| \nonumber \\
		&= \sum_{J \subseteq  I_{N,A}: |J| \geq 2}\left|\sum_{\pi \in B(A)} \prod_{m \in I_{N,A} \setminus J}\alpha_{m,D}(\pi_{\overline{\cP_m}}) \prod_{n \in J} \Delta \alpha_n(\pi_{\overline{\cP_n}}) \right|\nonumber\\
		&\leq \sum_{J \subseteq  I_{N,A}: |J| \geq 2}\sum_{\pi \in B(A)} \prod_{m \in I_{N,A} \setminus J}\alpha_{m,D}(\pi_{\overline{\cP_m}}) \prod_{n \in J} \left|\Delta \alpha_n(\pi_{\overline{\cP_n}}) \right|\nonumber\\
		&= \sum_{j \in I_{N,A}} \sum_{k=2}^{|I_{N,A}|} \sum_{J \subseteq  I_{N,A}: j \in J,  |J| = k} \frac{1}{k} \sum_{\pi \in B(A)} \prod_{m \in I_{N,A} \setminus J}\alpha_{m,D}(\pi_{\overline{\cP_m}}) \prod_{n \in J} \left|\Delta \alpha_n(\pi_{\overline{\cP_n}}) \right|\nonumber \\
		&= \sum_{j \in I_{N,A}} \sum_{\pi \in B(A)} \left|\Delta \alpha_j(\pi_{\overline{\cP_j}}) \right|
		\left( \sum_{k=2}^{|I_{N,A}|} \sum_{J \subseteq  I_{N,A}: j \in J,  |J| = k} \frac{1}{k}  \prod_{m \in I_{N,A} \setminus J}\alpha_{m,D}(\pi_{\overline{\cP_m}}) \prod_{n \in J: n\neq j} \left|\Delta \alpha_n(\pi_{\overline{\cP_n}}) \right|
		\right) \label{fJalpha}
		\end{align}
		The first equality follows from the definition of $f^J(A)$,
		and the second and the third equalities follow from suitable arrangement of indices for summations.
		Since \eqref{boundDeltaAlphaAbs} holds for every $n \in I_{N,A}$ and $\pi \in \{0,1\}^{\cP_n}$, the value in the parenthesis in \eqref{fJalpha} can be bounded by
		\begin{align*}
		&\sum_{k=2}^{|I_{N,A}|} \sum_{J \subseteq  I_{N,A}: j \in J,  |J| = k} \frac{1}{k}  \prod_{m \in I_{N,A} \setminus J}\alpha_{m,D}(\pi_{\overline{\cP_m}}) \prod_{n \in J: n\neq j} \left|\Delta \alpha_n(\pi_{\overline{\cP_n}}) \right|\\
		&\leq \sum_{k=2}^{|I_{N,A}|} \sum_{J \subseteq  I_{N,A}: j \in J,  |J| = k} \frac{1}{k}  \prod_{m \in I_{N,A} \setminus J}\alpha_{m,D}(\pi_{\overline{\cP_m}}) \prod_{n \in J: n\neq j} \left(\frac{1}{\sqrt{\lambda}N}\alpha_{n,D}(\pi_{\overline{\cP_n}})\right)\\
		&= \sum_{k=2}^{|I_{N,A}|} \sum_{J \subseteq  I_{N,A}: j \in J,  |J| = k} \frac{1}{k(\sqrt{\lambda}N)^{k-1}}  \prod_{m \in I_{N,A}: m \neq j}\alpha_{m,D}(\pi_{\overline{\cP_m}}) \\
		&= \sum_{k=2}^{|I_{N,A}|}  \frac{\binom{|I_{N,A}|}{k-1}}{k(\sqrt{\lambda}N)^{k-1}}  \prod_{m \in I_{N,A}: m \neq j}\alpha_{m,D}(\pi_{\overline{\cP_m}})\\
		&\leq \sum_{k=2}^{|I_{N,A}|}  \frac{ 1 }{k! \sqrt{\lambda}^{k-1}}  \prod_{m \in I_{N,A}: m \neq j}\alpha_{m,D}(\pi_{\overline{\cP_m}})\\
		&\leq \frac{ 1 }{\sqrt{\lambda}} \prod_{m \in I_{N,A}: m \neq j}\alpha_{m,D}(\pi_{\overline{\cP_m}}).
		\end{align*}
		The last inequality holds since $\lambda \geq 1$.
		Putting this into \eqref{fJalpha}, we have
		\begin{align}
		\eqref{fJalpha} &\leq \frac{ 1 }{\sqrt{\lambda}} \sum_{j \in I_{N,A}} \sum_{\pi \in B(A)} \left|\Delta \alpha_j(\pi_{\overline{\cP_j}}) \right|
		\prod_{m \in I_{N,A}: m \neq j}\alpha_{m,D}(\pi_{\overline{\cP_m}}) \nonumber \\
		&= \frac{ 1 }{\sqrt{\lambda}} \sum_{j \in I_{N,A}} \sum_{\pi \in \{0,1\}^{\cP_j} } \sum_{k=0,1} \left|\Delta \alpha_j(\overline{\pi}^{k}) \right| \sum_{\substack{ \pi' \in B(A): \\ \pi'_{\cP_j} = \pi, \pi_j = k } } \prod_{m \in I_{N,A} : m \neq j}\alpha_{m,D}(\pi'_{\overline{\cP_m}}) \nonumber \\
		&= \frac{ 1 }{\sqrt{\lambda}} \sum_{j \in I_{N,A}} \sum_{\pi \in \{0,1\}^{\cP_j} } \sum_{k=0,1} \left|\Delta \alpha_j(\overline{\pi}^{k}) \right| \beta'_{j}(\pi,k, A)  \nonumber \\
		&= \frac{ 1 }{\sqrt{\lambda}} \sum_{j \in I_{N,A}} \sum_{\pi \in \{0,1\}^{\cP_j} \setminus D^{\downarrow}_j} \left|\Delta \alpha_j(\overline{\pi}^{1}) \right| 
		\left(\beta'_{j}(\pi,0, A) + \beta'_{j}(\pi,1, A)  \right)\nonumber \\
		&\leq \frac{2 }{\sqrt{\lambda}}  \sum_{j \in I_{N,A}} \sum_{\pi \in \{0,1\}^{\cP_j} \setminus D^{\downarrow}_j} \left|\Delta \alpha_j(\overline{\pi}^{1}) \right|  \beta_{j,D}(\pi,A) \nonumber \\
		&= \frac{2 }{\sqrt{\lambda}}  \sum_{j \in I_{N,A}} \sum_{\pi \in \{0,1\}^{\cP_j} \setminus D^{\downarrow}_j} \beta_{j,D}(\pi,A) \left|\frac{1}{T_j(\pi) }\sum_{k=1}^{T_j(\pi)} Y_{j,\pi,k} - \alpha_j(\overline{\pi}^1) \right| \label{betaTY}
		\end{align}
		The second equality follows from the definition \eqref{dfBetaPra} of $\beta'_{j}(\pi,k, A)$,
		and the third equality holds since $\Delta \alpha_j(\overline{\pi}^{1}) = 0$ for $\pi \in D^{\downarrow}_j$ and $\Delta \alpha_j(\overline{\pi}^{1}) =  -\Delta \alpha_j(\overline{\pi}^{0})$ for every $j \in I_{N,A}$ and $\pi \in \{0,1 \}^{\cP_j}$.
		The inequality follows from \eqref{betaAbsBoundSub}.
		The last equality then follows from the definition of $\Delta \alpha_j(\overline{\pi}^{1})$ and \eqref{eq:alg2_hat_alpha}.
		Since $\Prob(T_j(\pi) \geq \tau_j(\pi)) \geq 1- 1/T$, by applying Proposition~\ref{prop_hoeffding} to independent variables $Y_{j,\pi,k}$ for $k=1,2,\dots,T_j(\pi)$ over $[0,1]$ with $\varepsilon = \sqrt{\log T / (2 \tau_j(\pi))}$, it holds that
		\begin{align}
		\Prob \left( \left| \frac{1}{T_j(\pi) }\sum_{k=1}^{T_j(\pi)} Y_{j,\pi,k} - \alpha_j(\overline{\pi}^1) \right| 
		\leq \sqrt{\frac{   \log T}{2 \tau_j(\pi)}} \right) \geq 1 - \frac{3}{T}. \label{TYProbBound}
		\end{align}
		Thus, by \eqref{betaTY} and \eqref{TYProbBound}, the following holds for every $A \in \cA$ with probability at least $1 - 3C/T$:
		\begin{align*}
		\sum_{J \subseteq  I_{N,A}: |J| \geq 2} \left| f^J(A) \right|
		&\leq  \sum_{j \in I_{N,A}} \sum_{\pi \in \{0,1\}^{\cP_j}} \sqrt{\frac{2 \beta^{2}_{j,D} (\pi,A) \log T}{\lambda \tau_j(\pi)}} \\
		&\leq  \sum_{j \in I_{N,A}} \sum_{\pi \in \{0,1\}^{\cP_j}} \sqrt{\frac{ 8\beta^2_{j,D}(\pi,A) \log T}{ 3\lambda T_{j,\pi}}} \\
		&\leq \sum_{j \in I_{N,A}} \sum_{\pi \in \{0,1\}^{\cP_j}} \sqrt{\frac{ 8\beta^2_{j,D}(\pi,A) C \log T}{ \lambda \beta_j(\pi,\hat{A}_{j,\pi}) T}}. 
		\end{align*}
		The second inequality follows from \eqref{lem_bound21},
		the third inequality follows from \eqref{Tnpi}.
		Since $\max_{A \in \cA} \beta_{j,D}(\pi,A) \leq e \max _{A \in \cA} \hat{\beta}_{j}(\pi,A) = e\hat{\beta}_{j}(\pi,\hat{A}_{j,\pi}) \leq e^2\beta_{j}(\pi,\hat{A}_{j,\pi})$ by \eqref{beta_ratio1},
		we have 
		\begin{align*}
		\max_{A\in \cA} \sum_{J \subseteq  I_{N,A}: |J| \geq 2} \left| f^J(A) \right| &\leq \max_{A \in \cA} \sum_{j \in I_{N,A}} \sum_{\pi \in \{0,1\}^{\cP_j}} \sqrt{\frac{ 8\beta^2_{j,D}(\pi,A) C \log T}{ \lambda \beta_j(\pi,\hat{A}_{j,\pi}) T}}\\
		&\leq \max_{A \in \cA} \sum_{j \in I_{N,A}} \sum_{\pi \in \{0,1\}^{\cP_j}} \sqrt{\frac{ 8e^2\beta_{j,D}(\pi,A) C \log T}{ \lambda T}}\\
		&\leq \sqrt{\frac{ 8e^2 C^3 \log T}{ \lambda T}}
		\end{align*}
		The last inequality holds since $\beta_{j,D}(\pi,A) \leq 1$ and $\sum_{j \in I_{N,A}} |\{0,1\}^{\cP_j}| \leq C$.
		Thus, if \eqref{boundDeltaAlphaAbs} holds for every $n =1,2,\dots,N$ and $\pi \in \{0,1\}^{\overline{ \cP_{n}} }$, which is attained with a probability of at least $1 - (6C)/T$, then the desired probabilistic bound (ii) holds with a probability of at least $1 - (3C)/T$. This directly implies that the bound (ii) holds with a probability of at least $1-(9C)/T$.
	\end{proof}
\end{lem}
The above two lemmas imply Proposition~\ref{prop_bound} as follows.
\begin{proof}[Proof of Proposition~\ref{prop_bound}]
	Let $A\in \cA$.
	By Lemma~\ref{lem_f}, it holds that
	\begin{align*}
	\left|\mu(A)  - \hat{\mu}(A) \right| &\leq  \left|\mu_{D}(A)  - \hat{\mu}(A) \right| 
	+ (\mu(A) - \mu_{D}(A)) \\
	&\leq \left| \sum_{J \subseteq I_{N,A}: |J| \geq 1} f^{J} (A)\right| +  (\mu(A) - \mu_{D}(A))\\
	&\leq \left| \sum_{j \in I_{N,A}} f^{j} (A)\right|
	+  \sum_{J \subseteq I_{N,A} : |J| \geq 2} |f^{J} (A)|  + (\mu(A) - \mu_{D}(A)).
	\end{align*}
	
	Since \eqref{beta_ratio1}, \eqref{alphabeta1}, and \eqref{beta_diff} hold,
	the first term and the second term are respectively bounded by Lemma~\ref{lem_bound2}(i) and (ii) with probability $1 - (C+2)/T$ and $1 - 9C/T$, as follows:
	\begin{align*}
	\left| \sum_{j \in I_{N,A}} f^{j} (A)\right| \leq \sqrt{  \frac{2 e^6 \gamma^* \log (|\cA|T)}{ T}}, \\
	\sum_{J \subseteq  I_{N,A}: |J| \geq 2} |f^J(A)|
	\leq \sqrt{ \frac{8 e^2 C^3 \log T}{\lambda T}}.
	\end{align*}
	The third term, then, is bounded by \eqref{muD_diff} as
	\begin{align*}
	\mu(A) - \mu_{D}(A) \leq 8e^2 (C^3 + C) S(\lambda).
	\end{align*}
	Thus we have the desired bound with probability at least $1 - (C+2 + 9C)/T = 1 - (10C+2)/T$.
\end{proof}

\subsection{Proof of Theorem~\ref{thm_main} and Proposition~\ref{prop_gamma_bound}}\label{subsec_algo3}
We prove our main result on the basis of Propositions~\ref{prop_est} and~\ref{prop_bound}.
\begin{proof}[Proof of Theorem~\ref{thm_main}]
	We put $\lambda = C^3/N$. From Propositions~\ref{prop_gamma_bound},~\ref{prop_est}, and~\ref{prop_bound}, the following holds for every $A \in \cA$ with probability at least $1 - (16C+2)/T$:
	\begin{align*}
	\left|\mu(A)  - \hat{\mu}_{\hat{D}}(A) \right| 
	&\leq \sqrt{\frac{2e^6 \gamma^* \log (|\cA|T) }{ T}} + \sqrt{ \frac{8e^2 N\log T}{T}} + 8e^2(C^3 + C ) S(C^3/N)\\
	&\leq \sqrt{\frac{8 e^6 \max\{\gamma^*, N \} \log (|\cA|T) }{ T}} + \frac{192 e^2 N C^7 \log T}{T}.
	\end{align*}
	Let us define $A^* = \argmax_{A \in \cA} \mu(A)$ and $\hat{A} = \argmax_{A \in \cA} \hat{\mu}_{D}(A)$. It then holds that
	\begin{align}
	\mu^* - \mu(\hat{A}) &= \mu(A^*) - \hat{\mu}_{D}(A^*) + \hat{\mu}_{D}(A^*) - \mu(\hat{A}) \\
	&\leq  |\mu(A^*) - \hat{\mu}_{D}(A^*)| + |\mu(\hat{A}) - \hat{\mu}_{D}(\hat{A})| \\
	&\leq \sqrt{\frac{32 e^6 \gamma^* \log (|\cA|T) }{ T}} + \frac{384 e^2  N C^7 \log T}{T}. \label{bound_reg}
	\end{align}
	The first inequality holds since $\hat{A}$ is the maximizer of $\hat{\mu}_D$ and thus $ \hat{\mu}_D(A^*) \leq \hat{\mu}_D(\hat{A})$.
	Thus the difference between $\mu^*$ and $\mu(\hat{A})$ is bounded by \eqref{bound_reg} with a probability of at least $1 - (16C+2)/T$, which implies the desired regret bound:
	\begin{align*}
	R_T &\leq \sqrt{\frac{32 e^6 \gamma^* \log (|\cA|T) }{T}} + \frac{384 e^2 N C^7 \log T + 16C + 2}{T}\\
	&=O\left( \sqrt{\frac{ \max\{\gamma^*, N\} \log (|\cA|T) }{ T}} \right).
	\end{align*}
\end{proof}
We conclude this section 
by proving Proposition~\ref{prop_gamma_bound}, whose proof is independent of the above series of discussions.
\begin{proof}[Proof of Proposition~\ref{prop_gamma_bound}]
	We first show the left inequality.
	It holds that
	\begin{align*}
	\gamma^* = &\min_{\eta \in [0,1]^{\cA}}  \max_{A \in \cA} \sum_{ n \in I_{N,A}} \sum_{\substack{ \pi \in  \{0,1\}^{\cP_n} \\ : \beta_{n}(\pi, A) > 0 }}  \frac{\beta_{n}^2(\pi, A)}{\sum_{A' \in \cA} \eta_{A'} \beta_n(\pi,A')} \quad \text{s.t. } \sum_{A' \in \cA} \eta_{A'} = 1\\
	\geq   &\max_{A \in \cA} \sum_{n \in I_{N,A}} \min_{\eta \in [0,1]^{\cA}} \sum_{\substack{ \pi \in  \{0,1\}^{\cP_n} \\ : \beta_{n}(\pi, A) > 0 }}  \frac{\beta_{n}^2(\pi, A)}{\sum_{A' \in \cA} \eta_{A'} \beta_n(\pi,A')} \quad \text{s.t. } \sum_{A' \in \cA} \eta_{A'} = 1 \\
	\geq   &\max_{A \in \cA} \sum_{n \in I_{N,A}} \min_{x_{n} \in [0,1]^{\cP_n}}\left( \sum_{\substack{ \pi \in  \{0,1\}^{\cP_n} \\ : \beta_{n}(\pi, A) > 0 }}  \frac{\beta_{n}^2(\pi, A)}{x_{n,\pi}}
	\quad \text{s.t. } \sum_{\pi \in \{0,1\}^{\cP_n}} x_{n,\pi} = 1\right),
	%	= &\max_{A \in \cA} \sum_{n \in [1,N]: A_n = *} \sum_{\substack{ \pi \in  \{0,1\}^{\cP_n} \\ : \beta_{n}(\pi, A) > 0 }} \beta_{n} (\pi, A) =  N -\min_{A \in \cA} |A|.
	\end{align*}
	The last inequality holds by letting $x_{n, \pi}=\sum_{A' \in \cA} \eta_{A'} \beta_n(\pi,A')$, noting that
	\[\sum_{\pi \in \{0,1\}^{\cP_n}} x_{n,\pi} = \sum_{\pi \in \{0,1\}^{\cP_n}}\sum_{A' \in \cA} \eta_{A'} \beta_n(\pi,A') = \sum_{A' \in \cA} \eta_{A'} \sum_{\pi \in \{0,1\}^{\cP_n}} \beta_n(\pi,A') = 1.\]
	For each $A \in \cA$ and $n \in I_{N,A}$ the minimization problem with respect to $x_{n}$ is a convex optimization problem, where the minimum is attained when $x_{n, \pi} = \beta_n(\pi,A)$.
	Hence the above lower bound is equal to
	\[
	\max_{A \in \cA} \sum_{n \in [1,N]: A_n = *} \sum_{\substack{ \pi \in  \{0,1\}^{\cP_n} : \\  \beta_{n}(\pi, A) > 0 }} \beta_{n} (\pi, A) =  N -\min_{A \in \cA} |A|.
	\]
	This proves the left inequality.
	
	We next prove the right inequality.
	For each $n \in [1,N]$ and $\pi \in \{0,1\}^{\cP_n}$, let $A_{n, \pi} \in \cA$ be an intervention that attains ${\max}_{A' \in \cA} \beta_{n}(\pi,A')$.
	Note that, if such $A_{n, \pi}$ is not unique, then we choose one of them.
	Consider the solution $\eta'$ such that $\eta'_A = |\{ (n, \pi)\mid  n \in [1,N], \pi \in \{0,1\}^{\cP_n}, A = A_{n, \pi}\} |/C$ for each $A\in \cA$. 
	Let us confirm that, since the number of pairs $(n, \pi)$ is exactly equal to $C$, it holds that $\sum_{A \in \cA}\eta'_A = 1$, and thus $\eta'$ is in fact a feasible solution for the minimization problem.
	For each $(n,\pi)$, it holds that $\eta'_{A_{n,\pi}} \geq 1/C$, and thus we have
	\[
	\sum_{A' \in \cA} \eta_{A'} \beta_n(\pi,A')\geq \beta_n(\pi,A_{n, \pi})/C.
	\]
	Hence it holds that
	\begin{align*}
	\gamma^* &\leq \max_{A \in \cA} \sum_{n=1}^N \sum_{\substack{ \pi \in  \{0,1\}^{\cP_n} : \\  \beta_{n}(\pi, A) > 0 }} \frac{\beta_{n}^2(\pi, A)}{\sum_{A' \in \cA} \eta'_{A'} \beta_n(\pi, A')} \\
	&\leq \max_{A \in \cA} \sum_{n=1}^N \sum_{\substack{ \pi \in  \{0,1\}^{\cP_n} : \\  \beta_{n}(\pi, A) > 0 }} \frac{\beta_{n}^2(\pi, A)}{\beta_n(\pi, A_{n, \pi}) / C} \\
	&\leq \sum_{n=1}^N \sum_{\substack{ \pi \in  \{0,1\}^{\cP_n} : \\  \beta_{n}(\pi, A) > 0 }} \beta_{n}(\pi, A_{n, \pi})C \leq NC.
	\end{align*}
	Also, consider $\eta''_A := 1/|\cA|$ for all $A \in \cA$.
	Then, since $\sum_{A' \in \cA} \eta_{A'} \beta_n(\pi,A')=\sum_{A' \in \cA} \beta_n(\pi,A')/|\cA|\geq \beta_n(\pi,A)/|\cA|$ for any $A\in \cA$, it holds that
	\begin{align*}
	\gamma^* &\leq \max_{A \in \cA} \sum_{n=1}^N \sum_{ \pi \in  \{0,1\}^{\cP_n}} \frac{\beta_{n}^2(\pi, A)}{\beta_n(\pi, A) / |\cA|} \\
	&\leq \max_{A \in \cA} \sum_{n=1}^N \sum_{ \pi \in  \{0,1\}^{\cP_n}} \beta_{n}(\pi, A) |\cA| \leq N|\cA|.
	\end{align*}
	Thus we have the statement.
\end{proof}

\begin{figure*}[htb]
	\centering
	\includegraphics[width=\linewidth]{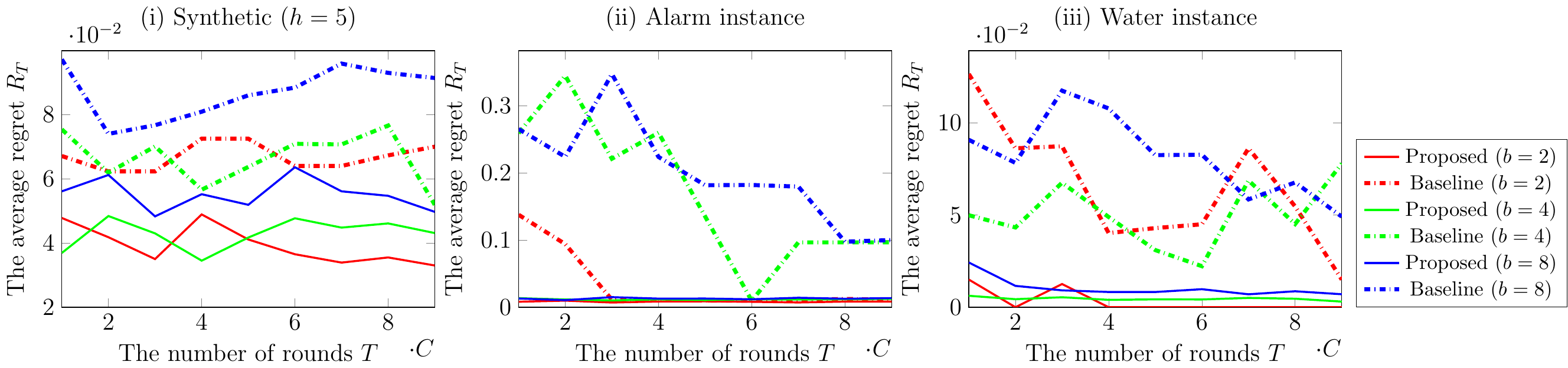}
	\caption{Average regret over synthetic and real-world instances}
	\label{fig:result}
\end{figure*}

\section{Experiments}\label{sec_exp}

We now demonstrate the performance of the proposed algorithm through experimental evaluations and compare it with a baseline algorithm~\cite{audibert2010best} which was proposed for the general best arm identification problem and thus cannot take advantage of given side-information of causal graph structure.

\subsection{Instances}
We evaluated the algorithms on both synthetic and real-world instances.
Recall that an instance of the causal bandit problem consists of a DAG $G$, an intervention set $\cA$, and $\alpha_n (n=1,\ldots,N)$.

In the synthetic instances, 
the DAG $G$ is defined as a directed complete binary tree of height $4$, where each edge is oriented toward the root.
From this construction, the number of nodes is $N=2^6-1=31$, which consists of $2^4$ leaves (nodes without incoming edges) and $(2^4-1)$ non-leaves.
Then the number of uncertain parameter is $C = 2^2 \times (2^4-1) + 2^0 \times 2^4=60$ in these instances.

In the real-world instances,
the DAG $G$ is constructed from the Alarm and the Water data sets
in a Bayesian Network Repository\footnote{\url{http://www.cs.huji.ac.il/~galel/Repository/}}.
The numbers $N$ of nodes in the DAGs constructed from 
Alarm and Water data sets are 37 and 32, respectively.
We consider interventions on nodes that have no incoming edge.
$C = 116$ for the instances with Alarm data set,
and $C=248$ for the instances with the Water data set.

For each $G$,
we consider interventions over all leaves
which fixes exactly $b \in \mathbb{N}$ nodes as $1$ and the others as $0$.
We call this parameter $b$ budget,
and the number of intervention $|\cA|$ is then controlled by
the budget.
In both of the synthetic and the real-world instances,
the budget $b$ varies among $\{2,4,8\}$.
In the synthetic instances, the numbers of interventions $|\cA|$ 
are 120, 1820, and 12870 for $b \in \{2,4,8\}$, respectively.
In the instances with the Alarm data set, the numbers of interventions $|\cA|$ are 78, 793, and 3796 for $b \in \{2,4,8\}$, respectively.
In the instances with the Water data set, the numbers of interventions $|\cA|$ are 36, 126, and 256 for $b \in \{2,4,8\}$, respectively.

For each $n \in \{1,\ldots,N\}$ and $\pi \in \{0,1\}^{\overline{\cP_n}}$, 
we generate $\alpha_n(\pi)$ from the uniform distribution over $[0,1]$.

For each of those instances, we executed the algorithms 10 times and compared their average regrets.

\subsection{Implementation of the proposed algorithm}
Our algorithm given in Section~\ref{sec_proposed} is designed conservatively 
to obtain the theoretical regred bound~(Theorem~\ref{thm_main}),
and there is a room to modify the algorithm to be more efficient in practice
although the theoretical regret bound may not hold for it.
In our implementation, 
we introduced the following three modifications into the proposed algorithm.
First, while Algorithm~\ref{algo_est2} discards samples obtained 
for computing 
$\check{\alpha}'$
in Algorithm~\ref{algo_est1} to maintain the independence between $\hat{\beta}$ and $\hat{\alpha}$,
we use all of them also in Algorithm~\ref{algo_est2} in our implementation.
Next, we ignore the truncation mechanism of Algorithm~\ref{algo_est1}
by setting $\lambda=0$.
We expect these two modifications 
make the estimates of the algorithm more accurate.
Finally, instead of solving \eqref{optProb},
we set $\eta_A$ by
$\eta_{A} = 1/C$ if $A = \hat{A}_{n,\pi}$ 
for some $n \in [1,N]$ and $\pi \in \{0,1\}^{\cP_n}$,
and $\eta_A = 0$ otherwise.
Since it is time-consuming to solve \eqref{optProb},
this modification makes the algorithm faster.

\subsection{Experimental results}
%To confirm the performance,
Figure~\ref{fig:result}(i) shows the average regrets over the synthetic instances
against the number of rounds $T \in \{C, 2C, \ldots,9C\}$.
Figures~\ref{fig:result}(ii) and (iii) respectively illustrate the average regrets for the 
real-world instances constructed from 
the Alarm and the Water data sets.

The results show that the proposed algorithm outperforms the baseline in every instance.
In particular, the gap is remarkably large ($>0.2$) in the Alarm data set (ii) with a large number of interventions ($b=4,8$, corresponding to $|\cA|=793,3796$, respectively,) and a small number of samples ($T \leq 4C=464$).
In these cases, the baseline cannot apply every intervention at least once.
On the other hand, the regret of the proposed algorithm only grows slowly with respect to the number of arms $|\cA|$, in all instances.
Thus the proposed algorithm provides effective regret, even when the number of interventions $|\cA|$ is $30$ times larger than the number of experiments $T$.

\section{Conclusion}
In this paper, we proposed the first algorithm for the causal bandit problem, where existing algorithms could deal with only localized interventions,
and proved a novel regret bound $O(\sqrt{\gamma^* \log (|\cA| T) / T})$ which is logarithmic with respect to the number of arms. 
Our experimental result shows that the proposed algorithm is applicable to systems where the number of interventions $|\cA|$ is much larger than $T$. One important future research direction would be to prove the gap-dependent bound as~\cite{sen2017identifying} has proven for localized interventions. Another research direction, which is mentioned in \cite{lattimore2016causal}, would include incorporation of a causal discovery algorithm to enable the estimation of the structure of a causal graph, which is currently assumed to be known in advance.

\bibliographystyle{plain}

\end{document}